\documentclass{article}

\usepackage{amsmath,amsfonts,bm}

\def\eqref#1{equation~\ref{#1}}

\def\1{\bm{1}}

\def\rmP{{\mathbf{P}}}

\def\rmS{{\mathbf{S}}}

\def\rmW{{\mathbf{W}}}

\def\mS{{\bm{S}}}

\DeclareMathAlphabet{\mathsfit}{\encodingdefault}{\sfdefault}{m}{sl}
\SetMathAlphabet{\mathsfit}{bold}{\encodingdefault}{\sfdefault}{bx}{n}

\usepackage{preamble}

\begin{document}
\maketitle


\begin{abstract}
The study of representations is widespread across fields, including neuroscience, psychology, and artificial intelligence. While representations are often studied and compared through similarities between stimuli, current methods provide only limited access to the dimensions that shape these representations and are often limited in interpretability. To overcome these challenges, here we introduce Similarity-Based Representation Factorization (SRF), a general computational method for recovering low-dimensional, non-negative, interpretable embeddings from similarity matrices derived from measured data. Across simulations and many neural, behavioral, and computational datasets, SRF recovers interpretable dimensions from diverse forms of representational data, even for very sparsely sampled, incomplete data. The dimensions derived from these datasets match those obtained by task-specific models, predict independent behavioral properties, improve exploratory analysis, and offer higher power for confirmatory hypothesis testing than comparing similarity matrices. Together, these results establish SRF as a general-purpose method with broad applications for uncovering, understanding, and using the dimensions underlying representations.
\end{abstract}

\section{Introduction}

A central goal in neuroscience, psychology, and artificial intelligence is to understand how biological and artificial systems organize sensory input into representations that support flexible behavior. Influential theoretical and methodological accounts have formalized representations through their geometry, with stimuli represented as points in a multidimensional space whose distances reflect similarity relationships between them~\citep{Torgerson1952,Shepard1957,Shepard1962,Nosofsky1986, Kriegeskorte2013, Scholkopf2002, Kornblith2019, Gardenfors2000}. Since representational spaces are latent and typically high-dimensional, a core challenge has been to infer their representational structure from data, that is, to recover the properties -- or dimensions -- along which stimuli are organized and to compare representations across systems. Empirical work has approached this challenge in two complementary ways. One line of work has used data-driven approaches to summarize representational data or recover low-dimensional structure, including multidimensional scaling \citep[MDS,][]{Torgerson1952}, principal component analysis \citep[PCA,][]{Jolliffe1986}, non-negative matrix factorization \citep[NMF,][]{Lee1999}, and t-distributed stochastic neighbor embedding \citep[t-SNE,][]{Maaten2008}. A second line of work has tested whether observed similarity relationships match those predicted by a representational hypothesis. In this context, representational similarity analysis \citep[RSA,][]{Kriegeskorte2008-01, Kriegeskorte2013,Schutt2023} has become particularly influential, using similarity matrices as a common format for comparing representations across species, modalities, and systems.

Despite these advances, existing approaches remain limited in revealing the dimensions that give rise to representational structure. For example, knowing that a dog has a more similar representation to a cow than to a car does not reveal whether this similarity is driven by \textit{animacy}, \textit{naturalness}, \textit{real-world size}, \textit{shape}, or other, potentially unknown properties. Likewise, RSA can test whether a predefined hypothesis explains observed similarities, but it does not recover the dimensions underlying them. Methods such as MDS, PCA, and t-SNE can summarize representational data or similarity relations in low-dimensional spaces, but their axes are often difficult to interpret as meaningful representational properties beyond the first few dimensions. NMF can yield parts-based, interpretable components, and related formulations can factorize similarity structure directly \citep{Ding2005, Kuang2012}. Yet this line of work has not been translated into a general empirical framework for representational data, where the number of representational dimensions is usually not known in advance, and similarity matrices are often sparse or selectively sampled. Existing approaches that operate at the level of similarities therefore do not jointly solve the problems of recovering interpretable dimensions, determining how many dimensions are needed, and learning from incomplete data.  

These limitations are consequential because interpretable dimensions are frequently the object of scientific interest. In neuroscience, such dimensions can support theories about brain mechanisms \citep{Naselaris2011, Kriegeskorte2021, Bracci2023}, identify stimulus properties that may drive neural responses \citep{Contier2024, Khosla2022, Sorscher2022, Jagadeesh2022, Bao2020}, and explain behavior \citep{Roads2024-01, Hebart2020}. In artificial intelligence and machine learning, they can reveal where and why models fail \citep{Geirhos2021,Geirhos2024,Lapuschkin2019,Samek2021}, define what individual deep neural network units encode \citep{Bau2017,Olah2017}, and pinpoint where models align with or diverge from human perception \citep{Sucholutsky2025,Mahner2025, Du2025, Mahner2026}.

Here we address existing challenges in recovering interpretable dimensions from representational data by introducing Similarity-Based Representation Factorization (SRF), a general method for recovering interpretable dimensions directly from similarity relationships. SRF operates on full or partially sampled similarity matrices, learns from observed entries without imputing missing values, and estimates dimensionality using a cross-validation procedure designed specifically for recovering dimensions from similarity data. We first validate SRF in simulations with ground truth and then apply it across human behavioral judgments, electrophysiology, neuroimaging, semantic association data, and deep neural network representations. Across these settings, SRF recovers interpretable dimensions even from highly sparse and undersampled similarity data. These dimensions support both exploratory discovery and confirmatory hypothesis testing and afford greater statistical power than the widely used RSA framework, particularly under sparsity and higher noise. Together, these results establish SRF as a general method for moving from descriptions of representational similarities to identifying the dimensions that organize them.

\begin{figure}[]
    \centering
    \includegraphics[width=1.0\linewidth]{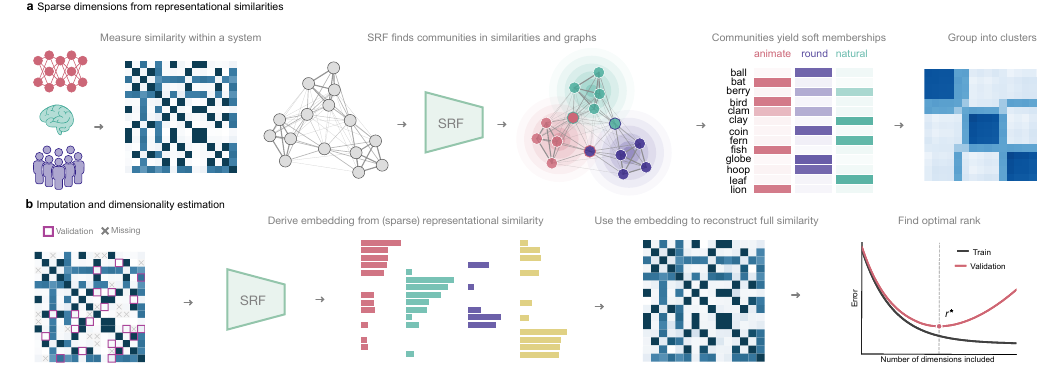}
    \caption{\textbf{Deriving non-negative, interpretable dimensions from similarity matrices.}
    \textbf{a},~Pairwise similarities between items can be represented as a matrix or, equivalently, as a weighted graph where edges connect similar items. Our method decomposes this similarity matrix into a small number of latent dimensions. In the illustrative graph, each color-coded community corresponds to a discovered dimension. Each item receives a non-negative loading on every dimension, yielding soft memberships. For example, \textit{lion} loads strongly on the \textit{animate} dimension, while \textit{ball} loads on both \textit{round} and \textit{natural}. These memberships reorganize the original similarity matrix into interpretable clusters, grouping items with soft contributions to one or more dimensions. \textbf{b},~To estimate dimensionality, observed similarity entries are split into training and held-out validation sets; the held-out entries are treated as missing values during training. Using our model, we factorize the training entries into an embedding with $r$ dimensions.
    Using the dot product similarity on this embedding gives a fully sampled similarity matrix that also predicts the held-out (or missing) entries. The optimal rank $r^*$ is selected at the minimum of the validation error curve.
    }
    \label{fig:overview}
\end{figure}

\section{Results}\label{sec:result}

SRF is a general-purpose method for deriving interpretable dimensions from representational data. Its input is a symmetric, non-negative similarity matrix, derived from feature-based measurements, such as neural responses or artificial neural network activations, or measured directly, as is the case for behavioral similarity judgments or association data. This makes SRF applicable to common data types in cognitive science, neuroscience, machine learning, and other fields that study representations.

SRF is built around symmetric non-negative matrix factorization \citep[sNMF,][]{Ding2005, Kuang2012}. We chose this formulation because it provides several useful properties for recovering interpretable dimensions from similarities. First, by operating on similarity matrices, it avoids the more complex problem of recovering all dimensions of the original feature space. Depending on the similarity measure, this can make the factorization invariant to transformations such as rotations, reflections, or translations in the original representation, thereby emphasizing stable relational structure between stimuli over arbitrary or noisy variation in the measured feature space. This property also allows SRF to operate on similarities regardless of whether they were computed from feature representations or acquired directly. Second, the symmetric factorization yields a single embedding matrix, which simplifies interpretation because the same dimensions both define the embedding and reconstruct the similarity matrix. Third, the non-negativity constraint encourages additive, naturally sparse dimensions and supports parts-based factors that are often easier to interpret \citep{Hebart2020, Fyshe2015}. Finally, the recovered factors are reasonably stable under modest dimensionality misspecification (see Supplementary Information~\ref{subsec:misspecify_k}), suggesting that small misspecification does not qualitatively alter the inferred structure.

Despite these benefits, sNMF has seen surprisingly little use in the study of real-world representational data, largely because existing formulations have not addressed two critical challenges. First, the dimensionality of the embedding is usually unknown. Choosing it appropriately matters because too few dimensions may merge distinct representational factors or miss important structure, whereas too many dimensions may split coherent factors or overfit noise. Second, many similarity datasets are incomplete because exhaustive sampling of all stimulus pairs is infeasible or because of missing data. We address both problems with SRF by developing an optimization procedure that fits the model while treating arbitrary entries as unobserved. This makes it possible to learn from incomplete similarity matrices without imputing missing values. This missing-data formulation is also what enables cross-validation for determining the underlying dimensionality. Specifically, for each candidate dimensionality, a restricted set of similarity entries is held out, treated as missing during training, and then predicted from the learned embedding. Because entries in a similarity matrix are statistically dependent, they cannot be held out arbitrarily without compromising the validation procedure. We therefore developed a cross-validation protocol using a restricted hold-out scheme that allows dimensionality to be evaluated by generalization to unseen similarities (Fig.~\ref{fig:overview}b; see Methods).

\subsection{Recovering dimensions from simulated data}

\begin{figure}[]
    \centering
    \includegraphics[width=1.0\linewidth]{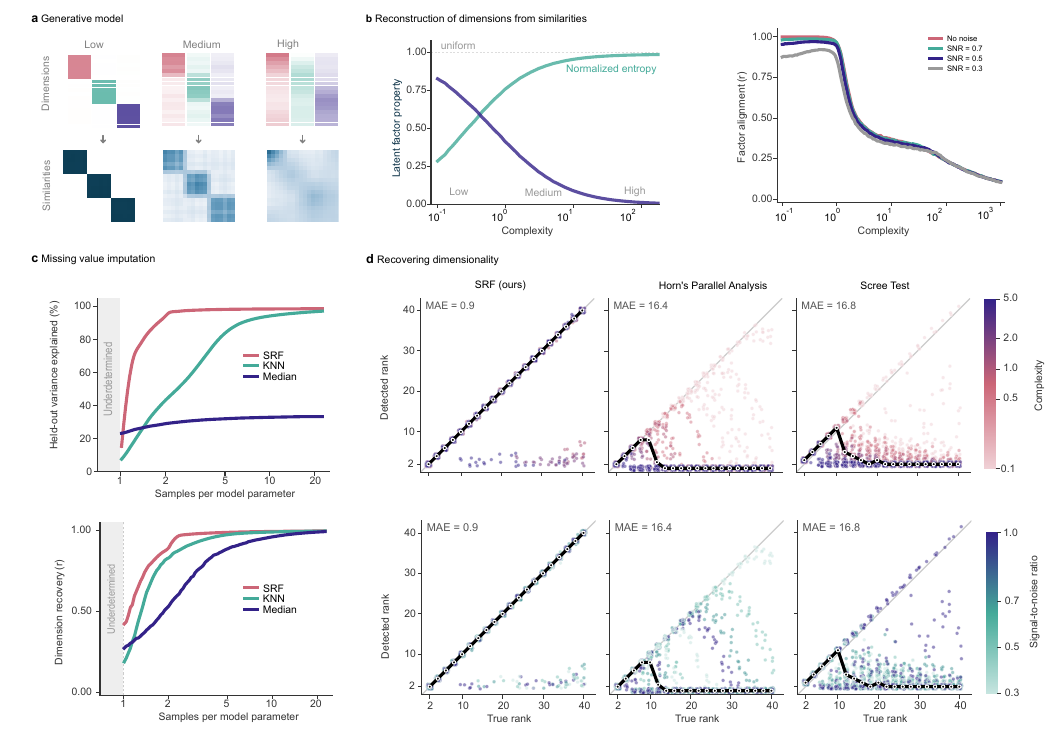}
    \caption{\textbf{Recovering latent dimensions from simulated representations.}
    \textbf{a},~Ground-truth embeddings were generated from a Dirichlet model with concentration parameter $\alpha$ controlling representational complexity. Low $\alpha$ produces sparse, cluster-like memberships where each item belongs to few dimensions, while high $\alpha$ produces diffuse, overlapping memberships where items spread across many dimensions. The top row shows the resulting factor matrices and the bottom row the corresponding similarity matrices for three complexity levels.
    \textbf{b},~Left: properties of the simulated ground-truth embeddings as a function of $\alpha$, measured by Hoyer sparsity and normalized entropy. Right: factor alignment, the correlation between recovered and ground-truth dimensions after optimal matching, shown for different noise levels.
    \textbf{c},~Our model achieves high reconstruction performance on missing similarities. The top panel shows the fraction of held-out variance explained and the bottom panel shows factor recovery, both as a function of observed samples per model parameter. The gray region marks the under-determined regime where fewer similarities are observed than there are parameters to estimate. The model outperforms $k$-nearest neighbor and median imputation across all levels of missingness. \textbf{d},~Rank detection accuracy for three methods. Cross-validation (ours) selects the rank that best predicts held-out similarities. Parallel analysis retains dimensions with eigenvalues exceeding those of random data. The scree test selects the rank where eigenvalues level off. Each dot represents one simulation and the diagonal marks perfect detection. Simulations in the top row vary the sparsity of the ground-truth embedding (color) and in the bottom row the signal-to-noise ratio. Mean absolute error (MAE) is shown for each method.}
    \label{fig:simulation}
\end{figure}

We first used simulations with known ground-truth dimensions to test three core requirements of SRF: (1) the recovery of latent dimensions and their dimensionality, (2) the ability to learn from incomplete similarity matrices, and (3) the validity of inferences from our cross-validation approach. We generated synthetic similarity matrices from a small set of underlying dimensions (see Methods) and controlled representational complexity via a parameter $\alpha$. This parameter modulated the ground-truth embeddings (Fig.~\ref{fig:simulation}a), ranging from sparse, cluster-like dimensions at low $\alpha$ (high Hoyer sparsity, low entropy) to diffuse and overlapping dimensions at high $\alpha$ (Fig.~\ref{fig:simulation}b, left). SRF recovered the generating dimensions accurately at low and moderate complexity, with recovery declining as dimensions became more diffuse and overlapping. Since empirical similarity measurements are often noisy, we next tested robustness to measurement noise by perturbing the ground-truth dimensions. Recovery degraded only modestly as the signal-to-noise ratio decreased (Fig.~\ref{fig:simulation}b, right), indicating that SRF can recover latent dimensions reliably even under moderate noise levels.

Next, we tested whether SRF could recover dimensions from incomplete similarity data, which is a common setting in behavioral and association-based datasets. We quantified the degree of missingness as the ratio of observed similarities to model parameters, with values below 1 indicating an under-determined regime with fewer observations than parameters to estimate (gray region in Fig.~\ref{fig:simulation}c). We compared SRF against $k$-nearest neighbor and median imputation, two standard baselines for handling missing data. SRF predicted held-out similarities more accurately than both baselines across nearly all levels of missingness (Fig.~\ref{fig:simulation}c, top), with the same pattern for recovery of the underlying dimensions (Fig.~\ref{fig:simulation}c, bottom). Notably, applying $k$-nearest neighbor or median imputation before fitting yielded worse recovery than leaving the missing entries unobserved. This suggests that imputation can distort pairwise similarity structure and bias inference about the generating dimensions.

Finally, we tested whether SRF cross-validation could identify the correct number of underlying dimensions. We evaluated recovery across true dimensionalities, degrees of missing data, and signal-to-noise ratios. Across these settings, SRF identified the true dimensionality more accurately than the scree test and Horn's parallel analysis, two common cognitive-science methods (Fig.~\ref{fig:simulation}d). Together, these simulations show that SRF recovers the dimensions underlying representational similarities even under noise and substantial undersampling, and that SRF cross-validation provides a principled way to estimate dimensionality from held-out similarities.

\subsection{Deriving dimensions from similarities improves statistical power over RSA}

\begin{figure}[]
    \centering
    \includegraphics[width=1.0\linewidth]{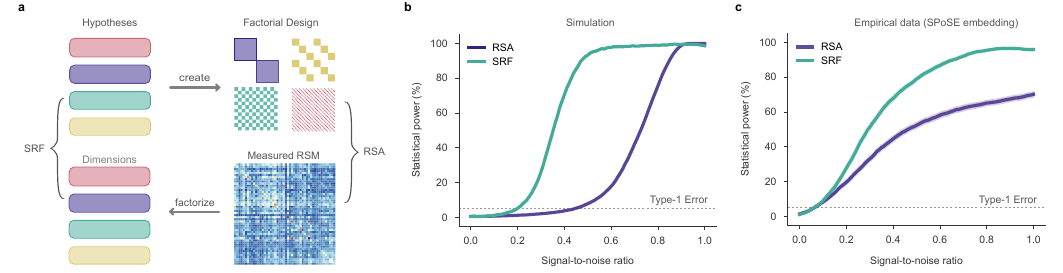}
    \caption{\textbf{SRF improves statistical power for hypothesis testing compared to RSA.}
    \textbf{a},~Schematic of the comparison. Four orthogonal categorical factors (colored bars) are encoded in a one-hot feature matrix $X$, yielding a clean similarity matrix $S = XX^\top$ that sums their contributions. This matrix contains the summed contributions of all factors and is then corrupted by Gaussian noise. RSA tests each hypothesis factor against the full similarity matrix, where the signal from any single factor is diluted by contributions from all other factors and noise. SRF first decomposes the similarity matrix into individual dimensions and then tests each hypothesis against its matched dimension, isolating the relevant signal.
    \textbf{b},~Statistical power, defined as the fraction of true factors correctly detected as significant, as a function of signal-to-noise ratio. SRF detects the true factors at substantially lower signal-to-noise ratios than RSA. The dashed line indicates the chance level at a significance threshold of $\alpha = 0.05$.
    \textbf{c},~Same comparison using dimensions sampled from the SPoSE semantic embedding \citep{Hebart2020}, which are sparse, non-negative, and correlated, providing a more stringent test than orthogonal factors. The advantage of SRF is even more pronounced, particularly at low signal-to-noise ratios.}
    \label{fig:rsa}
\end{figure}

In many applications, researchers aim not only to discover which dimensions structure representations, but also to test whether specific properties (e.g.\ \textit{animacy} or \textit{real-world size}) explain them. A widely used tool for confirmatory inference is RSA, which compares an observed similarity matrix with a model matrix that encodes a specific hypothesis, for example, a matrix in which objects from the same hypothesized category are more similar than objects from different categories. 

However, observed similarity matrices typically reflect variation along multiple properties at once. A hypothesis about a single property is therefore tested against a matrix shaped by many other properties and measurement noise, which can reduce statistical power. Since SRF decomposes similarity structure into distinct dimensions, this allows testing a hypothesis directly against the dimension that best matches it, rather than against the full similarity matrix (Fig.~\ref{fig:rsa}a). Consequently, this should provide a more sensitive statistical test. 

We first compared the statistical power of SRF-based hypothesis testing against RSA. We used a factorial design with four orthogonal dimensions (Fig.~\ref{fig:rsa}a, Methods). From this design, we generated pairwise similarities, added noise across a range of signal-to-noise ratios, and used SRF to derive dimensions from the noisy similarities. We then compared power for detecting each data-generating dimension using SRF against classical RSA, in which the corresponding property matrix (e.g.\ \textit{animacy}) was correlated with the full noisy similarity matrix. Across noise levels, SRF detected the data-generating dimensions with greater sensitivity than RSA (Fig.~\ref{fig:rsa}b). Note that dimension matching was carried out using cross-validation to avoid positive bias. Thus, decomposing noisy similarities into dimensions can increase the sensitivity of confirmatory hypothesis testing.

To assess whether this advantage extends to dimensions typically identified through data-driven discovery, we repeated the analysis using empirically derived, correlated dimensions. Specifically, we sampled dimensions from the SPoSE embedding \citep{Hebart2020}, which provides sparse, non-negative, and interpretable dimensions derived from human similarity judgments. Unlike the factorial design, these dimensions vary in their contribution to the overall similarity structure and are not statistically independent. Applying the same noise perturbation and testing procedure, we found that SRF again outperformed RSA, even more strongly than in the factorial design (Fig.~\ref{fig:rsa}c).

To understand this advantage, we examined statistical power as a function of how much variance each dimension contributed to the overall similarity matrix (Supplementary Fig.~\ref{fig:variance_quartile}). RSA had low power for dimensions that contributed little variance, because these dimensions contribute weakly to the full similarity matrix against which hypotheses are tested. In contrast, SRF retained high sensitivity by isolating these dimensions from the overall similarity structure. Thus, dimensions that would be difficult to detect with RSA can become accessible through data-driven decomposition. Together, these results indicate that decomposing similarity structure into dimensions can substantially increase statistical power for confirmatory dimension-specific hypothesis testing, while enabling the data-driven discovery of low-variance dimensions otherwise obscured in the full similarity matrix.

\subsection{SRF reveals stable and interpretable dimensions across empirical datasets}

\begin{figure}[]
    \centering
    \includegraphics[width=1.0\linewidth]{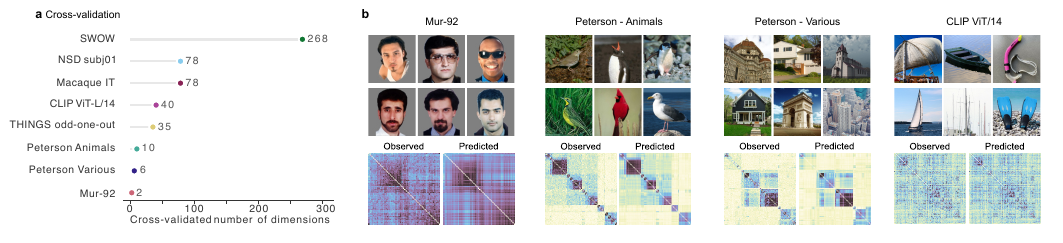}
    \caption{\textbf{Dimension discovery across empirical datasets.}
\textbf{a},~Optimal number of dimensions across diverse empirical datasets, obtained by the cross-validation procedure. \textbf{b},~Example dimensions (top) and observed vs.\ predicted similarity (bottom) across several datasets. Predicted similarities show that a small number of dimensions account for much of the variance in the original similarities.}
    \label{fig:visualization}
\end{figure}

Having established ground-truth recovery in simulations, we next tested whether SRF could recover stable and interpretable dimensions in empirical datasets spanning multiple domains. Specifically, we applied SRF in an exploratory fashion to empirical datasets covering directly measured behavioral similarities \citep[THINGS,][]{Hebart2023}\citep[Peterson-animals and Peterson-various;][]{Peterson2018}\citep[Mur-92,][]{Mur2013}, feature activations from a vision-language model \citep[CLIP-ViT,][]{Radford2021}, electrophysiological recordings from macaque inferior temporal cortex \citep{Papale2025}, fMRI responses from the Natural Scenes Dataset \citep[NSD,][]{Allen2022}, and free associations from the Small World of Words project \citep[SWOW,][]{DeDeyne2019} (see Supplementary Table~\ref{tab:datasets} for an overview). Unlike the other datasets, SWOW forms a sparse, incompletely sampled graph in which only 1.27\% of word pairs are connected (Methods).

Across datasets, SRF yielded similarity embeddings with dimensionality ranging from 2 to 268 (Fig.~\ref{fig:visualization}a and Supplementary Fig.~\ref{fig:cv_curves} for cross-validation curves). At these ranks, the embeddings accurately reconstructed the observed similarities (median $R^2 = 0.88$ across densely sampled datasets). For SWOW, where the observed associations formed an extremely sparse graph rather than a dense matrix, we evaluated reconstruction as link prediction, obtaining an AUC of 0.95, demonstrating that the recovered dimensions separated true associations from non-associations even under extreme sparsity. Dimensions were also stable across initializations. For each dataset, we re-ran SRF 30 times from random starts and measured split-half reliability across held-out items (Methods). Mean Spearman-Brown corrected reliability exceeded 0.82 for every dataset, indicating that the same dimensions reappeared across runs rather than reflecting variability specific to a single initialization.

Inspection revealed highly interpretable dimensions across behavioral, neural, and model-derived data (Fig.~\ref{fig:visualization}b). Stimuli loading most strongly on each dimension formed coherent groupings, often matching distinctions known from the cognitive and neural sciences. For example, in Mur-92, SRF recovered the animate/inanimate distinction, mirroring the well-established role of animacy as a major organizing principle of human object representation \citep{Konkle2013}. In Peterson-animals, it separated taxonomic groups such as birds and insects. In CLIP, it recovered perceptual dimensions, such as color, and semantic dimensions, such as animacy, in line with previous work \citep{Mahner2025}. In SWOW, SRF recovered interpretable dimensions spanning multiple levels of abstraction, from taxonomic groupings such as food-related items to abstract themes such as physical strength or affection (Fig.~\ref{fig:words}a,b). Together, these findings show that representational similarities from human judgments, neural recordings, deep neural networks, and association data can be summarized by stable and interpretable dimensions, allowing the same dimensions to be compared across brains, behavior, and AI.

\begin{figure}[t]
    \centering
    \includegraphics[width=1.0\linewidth]{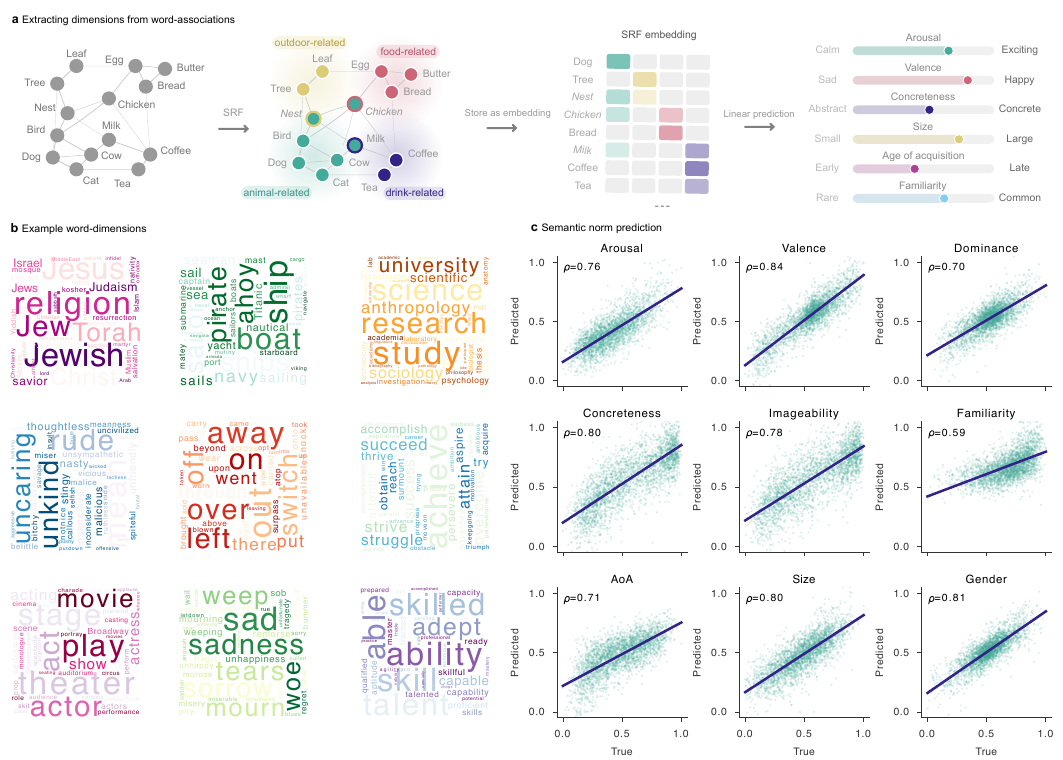}
    \caption{\textbf{SRF recovers interpretable semantic dimensions from word association data.}
    \textbf{a},~The Small World of Words dataset \citep{DeDeyne2019} provides free word associations that can be represented as a weighted graph (left). SRF decomposes the association similarity matrix into dimensions, shown as color-coded communities in the graph (middle). The resulting embedding is then used as input to cross-validated ridge regression to predict independent behavioral ratings (right).
    \textbf{b},~Example dimensions recovered by SRF, displayed as word clouds where word size is proportional to loading. The dimensions capture meaning at multiple levels of abstraction, from concrete categories such as alcoholic drinks or food to more abstract groupings such as physical strength or affection.
    \textbf{c},~Predicted versus actual behavioral ratings for nine properties from the Glasgow Norms \citep{Scott2019}, with Spearman correlations shown for each. AoA = age of acquisition.}
    \label{fig:words}
\end{figure}

\subsection{Sparse word associations yield semantic dimensions that predict independent behavior}

Having demonstrated that SRF recovers stable and interpretable dimensions across empirical datasets, we next asked whether these dimensions capture information that generalizes beyond the similarity data used to estimate them. This is important because a useful representation should support downstream inferences. We tested this in SWOW by using the recovered dimensions as features in models predicting independent behavioral ratings.

In SWOW, participants were presented with a cue word and reported three words as free associations. The dataset contains approximately 2.4 million responses across 8{,}647 cue words. Word associations are naturally represented as a sparse graph of semantic links than as a similarity matrix. We therefore converted the graph into pairwise similarities using positive pointwise mutual information \citep{Church1990}, which measures whether two words co-occur more often than expected by chance (Methods). The resulting similarity matrix was extremely sparse, with only 1.27\% of word pairs showing positive association, making it a stringent test of SRF.

We then tested whether these dimensions predicted independent behavioral ratings from the nine scales of the Glasgow Norms \citep{Scott2019}, available for a subset of SWOW words (Fig.~\ref{fig:words}a,c). The SRF embedding predicted all nine properties accurately, with correlations ranging from $\rho = 0.59$ (familiarity) to $\rho = 0.84$ (valence; Fig.~\ref{fig:words}c). Beyond the Glasgow Norms, the embedding also strongly predicted independent benchmarks of semantic similarity and relatedness (Supplementary Fig.~\ref{fig:similarity_benchmarks}).

Together, these downstream tests show that SRF dimensions are useful not only as interpretable summaries of representational structure, but also as compact, reusable features for predicting behavior in independent tasks.

\section{Discussion}

Understanding representations requires identifying the dimensions along which biological and artificial systems organize stimuli. We introduced Similarity-Based Representation Factorization (SRF) as a general method for recovering such dimensions from representational similarities, yielding non-negative, interpretable embeddings that can be estimated from incomplete data and whose dimensionality is selected by generalization to unseen data. Across simulations and real-world datasets spanning human behavior, neural recordings, semantic associations, and deep neural networks, SRF recovered stable dimensions that captured the core structure of measured representations. When applied to dimension-specific hypothesis testing, SRF yielded greater sensitivity than RSA, a dominant framework for targeted analyses of representational structure. By decomposing similarities into dimensions, SRF thus preserves the generality of similarity-based analysis while making the underlying factors accessible for interpretation, enabling more sensitive dimension-specific hypothesis tests, and supporting independent downstream prediction. 

Building on a long tradition of formalizing representations through geometry \citep{Torgerson1952,Shepard1957,Kriegeskorte2013}, SRF bridges two widely used approaches for representational analysis. RSA provides a powerful common format for comparing representations across systems using similarity matrices, but typically does not allow us to infer the dimensions that generate this similarity structure. Conversely, methods such as MDS, PCA, or NMF reveal structure in high-dimensional data, but standard applications often lack the combination of interpretable dimensions, principled dimensionality selection, and native support for incomplete similarity matrices. SRF bridges these approaches by operating directly on representational similarities while recovering interpretable dimensions that generate the representational structure. 

Interpretable, dimension-wise representations are increasingly used for mechanistic explanation in cognitive science, neuroscience, and machine learning. Non-negative embeddings have been used to contrast category-selective and distributed accounts of neural organization \citep{Khosla2022,Contier2024}, to identify interpretable dimensions underlying human similarity judgments \citep{Hebart2020}, to compare human, macaque, and deep network representations \citep{Mahner2025,Du2025,Mahner2026, vanBree2026}, and to build encoding models of neural representations \citep{Zeng2025,Klindt2025}. In machine learning, sparse autoencoders similarly decompose network activations into interpretable units that can be inspected and manipulated \citep{Bricken2023,Cunningham2024}. SRF generalizes this dimension-wise perspective to representational similarities, making it possible to move beyond scalar comparisons of similarity toward an explanation of which factors determine the underlying representation.

Future work could extend SRF in several directions. While linear readouts are widely used as a proxy for representational content \citep{DiCarlo2012,Majaj2015}, this formulation may miss similarity structure that depends on nonlinear interactions among dimensions. Moreover, representational spaces may contain hierarchical organization, with broad categorical distinctions scaffolding finer-grained features \citep{Costa2025,Muttenthaler2025}, and extending SRF to derive hierarchically structured dimensions could improve understanding \citep{Nickel2017,Khrulkov2020}. SRF also currently learns dimensions separately for each dataset, so cross-system comparisons require post-hoc alignment; joint factorizations could instead estimate shared and system-specific dimensions directly \citep{Thasarathan2025}. Finally, since representational similarities can be viewed as weighted graphs, SRF may generalize to other domains where pairwise relationships are sparse and partially observed, including protein interaction networks, brain connectivity graphs, or social networks. Its ability to learn from incomplete data may therefore make SRF useful beyond the representational domains studied here.

\section{Methods}

\subsection{Similarity-based representation factorization (SRF)}

SRF models  a non-negative similarity matrix $S \in \mathbb{R}_{\geq 0}^{n \times n}$ for $n$ stimuli using a low-rank symmetric factorization,

\begin{equation}
S \approx W W^\top,
\end{equation}

where $W \in \mathbb{R}_{\geq 0}^{n \times r}$ and $r \ll n$. By construction, the reconstruction $WW^\top$ is symmetric and positive semi-definite. While empirical similarity matrices are not always positive semi-definite \citep{Laub2004, Laub2006}, the factorization remains robust when $\mS$ deviates mildly from these properties. Each row $w_i$ of $W$ represents item $i$ as a non-negative weight vector over $r$ dimensions, and the similarity between items $i$ and $j$ is approximated by $w_i^\top w_j$. As dimensions can only contribute positively to similarity, this supports an additive and parts-based representation, which can be interpreted more directly. Because $W$ is not normalized, columns of $W$ can differ in scale, and columns with larger norm contribute proportionally more to the reconstructed similarity.

When all similarities are observed, SRF estimates $W$ by minimizing the Frobenius reconstruction error between $S$ and $WW^\top$. In many applications, however, only a subset of pairwise similarities is observed. To accommodate this case, we define a symmetric binary mask $\mathcal{M} \in \{0,1\}^{n \times n}$ where $\mathcal{M}_{ij}=\mathcal{M}_{ji}=1$ if $S_{ij}$ is observed, and $0$ otherwise. We assume the diagonal is observed, so $\mathcal{M}_{ii}=1$. We then estimate $W$ by solving

\begin{equation}
\min_{W \ge 0} \frac{1}{2} \lVert \mathcal{M} \odot (S - W W^\top) \rVert_F^2,
\end{equation}
where $\odot$ denotes element-wise multiplication and $\lVert \cdot \rVert_F$ is the Frobenius norm. Thus, only observed similarities contribute to the loss. Unobserved entries are excluded from the fitting objective but, after fitting, the model still yields predictions $(WW^\top)_{ij}$ for those unobserved pairs. We use this observation for our cross-validated rank selection (see Supplementary Information~\ref{sec:cross-validation}). We optimize the objective using the alternating direction method of multipliers (ADMM) \citep{Boyd2011}. Full derivations, update rules, and implementation details are provided in Supplementary Information~\ref{sec:srf-optim}.

\subsection{Model fitting, dimensionality selection, and stability}

The appropriate dimensionality $r$ is typically unknown, and is often itself a question of interest. We therefore fit SRF across the range of plausible candidate ranks and select the rank that best predicts held-out observed similarities. For each candidate dimensionality, we partition the observed similarities into training and validation folds, fit SRF on the training entries, and evaluate mean squared prediction error on the held-out entries. Because entries in a similarity matrix are not independent, we cannot use a standard entrywise matrix cross-validation. Instead, we use a restricted hold-out scheme that we designed specifically for similarity data. The dimensionality $\hat r$ that minimizes the average prediction error across folds is selected, and SRF is then refit at that rank on the full data. We used 5 folds and 5 repeats (25 fits per candidate rank), and optimized SRF via ADMM with penalty $\rho = 3$, up to 10{,}000 iterations (200 outer $\times$ 50 inner) or until convergence.

To assess the stability of the recovered embeddings, we fit SRF repeatedly from random initializations. Since the SRF objective is non-convex, these runs can converge to different factorizations even for the same dimensionality. Some of this variability is due only to permutation ambiguity in the recovered dimensions: if $W$ is a solution, then $W P$ is an equivalent solution for any permutation matrix $P$. Beyond this, different initializations may also converge to distinct but similarly good local minima. We therefore align the recovered dimensions across runs using the Hungarian algorithm based on pairwise column similarities. We then select the most central run as the final solution. We report this representative run instead of averaging embeddings, since averaging need not correspond to a solution of the SRF objective and can distort the reconstructed pairwise similarities. To quantify dimension stability across repeated fits, we compute cross-validated split-half reliability. For each random split of objects into two halves, dimensions are matched between pairs of runs on one half and evaluated on the other. Reliabilities are then averaged across run pairs and splits and corrected with the Spearman--Brown formula (see Supplementary Table~\ref{tab:datasets}). The final consensus embedding was computed from 30 random initializations, and split-half reliability was averaged over 100 random object splits.

In the rest of this section, we provide additional details on the rationale for developing the restricted hold-out scheme, and pointers to a full discussion in Supplementary Information~\ref{sec:cross-validation}. 

A useful property of the symmetric non-negative factorization used by SRF is that it is relatively insensitive to rank misspecification. The lower bound on the dimension recovery error grows only slowly with $\vert r - r^* \vert$ (see Supplementary Information~\ref{subsec:misspecify_k}). Modest misestimates of $r$ incur modest costs in reconstruction and prediction. Accurate rank selection nevertheless remains important for interpretability. If $r > r^*$, SRF can use the extra dimensions to model noise or sampling variation. Such dimensions may appear semantically coherent, even though they do not reflect stable structure in the similarity data.

The use of cross-validation for rank selection, in the context of factorizing low-rank similarity matrices, can lead to inflated estimates of rank. The entries of $S$ are not independent observations but interdependent (coupled) through the same low-dimensional structure. As a result, holding out individual entries does not necessarily create genuinely new prediction targets, unlike what would happen in a completely random full-rank matrix. Once the number of observed entries exceeds the matrix-completion threshold of $O(nr\,\mathrm{polylog}\,n)$ \citep{Candes2009, Gross2011}, the validation entries can be determined from the training entries with high probability (see Supplementary Information~\ref{subsec:info_leakage}). Validation error can therefore measure consistency with values already implied by the training data, rather than generalization to unseen similarities. This creates an information leakage problem: validation entries may no longer provide an independent test of prediction performance. It can bias rank selection toward overly large $r$, because higher-rank models may be rewarded for matching values that are already implied by the training entries. To address this problem, we use a restricted hold-out scheme that fixes an a priori sparse observation pattern before cross-validation. Through this, we keep each fold's training set below the matrix-completion threshold. Validation is then performed within this masked design, yielding a more conservative estimate of predictive performance. A detailed treatment of the leakage problem in standard alternatives is given in Supplementary Information~\ref{subsec:cv_protocol}.

\subsection{Constructing similarity matrices}
\label{sec:similarity_construction}

SRF operates on symmetric, non-negative similarity matrices. We constructed such matrices differently depending on the data source. For datasets that already provided direct pairwise judgments, we used the reported pairwise structure directly after symmetrizing and rescaling it to a non-negative similarity measure, if either step was necessary.

For triplet odd-one-out judgments, participants viewed three items and selected the item that was least similar to the other two. For each pair $(i,j)$, let $m_{ij}$ denote the number of triplets in which the pair was shown and $c_{ij}$ the number in which both items were retained as the most similar pair. We estimated pairwise similarity as
\[
s_{ij} = \frac{c_{ij} + \alpha}{m_{ij} + 2\alpha},
\]
with $\alpha = 1$ for mild Laplace smoothing. This yielded a symmetric similarity matrix bounded in $[0,1]$, with pairs that were never sampled treated as unobserved.

For word associations, we constructed a weighted directed graph from cue–response counts, removed self-loops and nodes without outgoing edges, and retained the largest strongly connected component. For each pair $(i,j)$, let $n_{ij}$ denote the number of times word $j$ was given in response to cue $i$, and let $N_{ij} = n_{ij} + n_{ji}$ be the symmetrized count. We estimated pairwise similarity as the Positive Pointwise Mutual Information \citep[PPMI;][]{Church1990},
\[
  s_{ij} = \max\left(\log_2 \frac{p_{ij}}{p_i\, p_j},\; 0\right)
\]
where $p_{ij} = N_{ij} / \sum_{kl} N_{kl}$, $p_i = \sum_j p_{ij}$, and $s_{ii} = -\log_2 p_i$. This yielded a sparse, symmetric, non-negative similarity matrix.

For continuous feature vectors, such as neural responses or deep neural network activations, we constructed similarities using kernels that produce non-negative matrices. For non-negative features, we used the linear kernel $k(\mathbf{x}_i, \mathbf{x}_j) = \mathbf{x}_i^\top \mathbf{x}_j$. For signed features, we used the radial basis function kernel
\begin{equation}
    k(\mathbf{x}_i, \mathbf{x}_j) = \exp\left(-\frac{\|\mathbf{x}_i - \mathbf{x}_j\|^2}{2\sigma^2}\right),    
\end{equation}

with bandwidth $\sigma = \alpha \cdot \tilde{d}$, where $\tilde{d}$ is the median pairwise Euclidean distance \citep{Scholkopf2002}. We set $\alpha = 0.4$ following previous work \citep{Mahner2026} showing that this multiplier provides a good estimate of the bandwidth for symmetric NMF, maximizing both factorization stability and reconstruction quality across diverse feature representations.

\subsection{Simulated data}

Ground-truth factor matrices $W \in \mathbb{R}_{\geq 0}^{n \times r}$ were generated by drawing each row independently from a symmetric Dirichlet distribution,
\begin{equation}
    \mathbf{w}_i \sim \mathrm{Dir}(\alpha, \ldots, \alpha),    
\end{equation}
so that rows summed to one. The concentration parameter $\alpha$ controlled representational complexity: small values produced sparse, cluster-like dimensions, whereas large values produced more diffuse and overlapping dimensions. Ground-truth similarity matrices were generated using a linear kernel $S = W W^\top$. When noise was included, we added Gaussian noise to $S$ and chose the noise scale by binary search so that the resulting matrix matched a prespecified signal-to-noise ratio (SNR), defined as the ratio of signal variance to total variance. Negative entries were clipped to zero and $S$ was symmetrized.

We used these simulations for three evaluations. For dimension recovery, we varied $\alpha$ and SNR and fit SRF at the true rank. Reconstruction quality was measured as explained variance between ground-truth and recovered similarity matrices, and dimension recovery as the correlation between dimensions after Hungarian matching. For missing-value recovery, we randomly masked off-diagonal entries at varying retention rates and compared SRF with $k$-nearest neighbor and median imputation. We quantified sampling density as the number of observed similarities per model parameter,

\begin{equation}
    \text{sampling density} = \frac{m_{\mathrm{obs}}}{nr},    
\end{equation}

where $m_{\mathrm{obs}}$ is the number of observed off-diagonal similarities and $nr$ is the number of free parameters in the latent embedding. Performance was evaluated on held-out similarities and on recovery of the ground-truth dimensions. For rank detection, we varied the true rank, complexity, and SNR, and compared SRF (with cross-validation) against two commonly used methods for determining dimensionality: parallel analysis and the scree test. Accuracy was quantified as the mean absolute error between the selected and true rank.

\subsection{Statistical power comparison}
We compared RSA and SRF in two simulation designs. In the factorial design, we crossed four orthogonal categorical factors with three or more levels each and encoded the resulting design in a one-hot feature matrix $X$. In the SPoSE design, each repeat sampled 50 objects and 5 dimensions from the SPoSE embedding \citep{Hebart2020}, yielding a continuous, sparse, and correlated feature matrix $X$. In both cases, clean similarity matrices were generated using a dot-product kernel $S = X X^\top$. We then added Gaussian noise and chose the noise scale by binary search so that the resulting matrix matched a prespecified signal-to-noise ratio (SNR). Negative entries were clipped to zero and matrices were symmetrized.

For RSA, each ground-truth column $x_j$ defined a hypothesis matrix $H_j = x_j x_j^\top$. Significance was assessed with a Mantel test \citep{Mantel1967} by correlating the upper-triangular entries of $H_j$ and $S$ and comparing the result with a null distribution obtained by jointly permuting the rows and columns of $S$. For SRF, we factorized $S$ at the true rank and matched recovered dimensions to ground-truth columns with the Hungarian algorithm using absolute correlations. To avoid selection bias in the matching step, we used leave-one-out alignment, in which the permutation for each object was determined from the remaining $n-1$ objects, and recomputed this alignment under each permutation when constructing the null distribution. All tests were one-sided with 1{,}000 permutations and Benjamini--Hochberg correction \citep{Benjamini1995} at $\alpha = 0.05$. Statistical power was defined as the fraction of ground-truth hypotheses detected as significant across 1{,}000 independent repeats.

\subsection{Empirical datasets}
We applied SRF to seven datasets spanning human similarity judgments, neural recordings, deep neural network activations, and semantic association data. The behavioral datasets comprised the Mur-92 dataset \citep{Mur2013}, which contains 92 images of faces, bodies, and objects collected with the multi-arrangement task, and two datasets from \citet{Peterson2018}, which provide pairwise similarity ratings for animal images and for images spanning diverse categories. For neural data, we used electrophysiological recordings from macaque inferior temporal cortex in response to 1{,}854 THINGS images \citep{Papale2025} and fMRI responses to natural scenes from the Natural Scenes Dataset \citep{Allen2022}. For model representations, we extracted image-encoder features from CLIP ViT-L/14 \citep{Radford2021} in response to 1{,}854 object images from THINGS-plus \citep{Stoinski2024}. For semantic association data, we used the Small World of Words English dataset \citep{DeDeyne2019}.

For datasets that already provided pairwise similarities, we used the reported similarity matrices directly. For the others, we describe the procedures to generate similarity matrices in detail in Section~\ref{sec:similarity_construction}, and report them in abridged fashion here. For the macaque electrophysiology dataset, we retained only channels with split-half reliability greater than 0.3 before constructing similarity matrices from them. For the remaining datasets, we constructed similarity matrices using the RBF kernel. For the SWOW dataset, we constructed a PPMI-based similarity matrix from association counts. All datasets were then factorized with SRF using the same fitting, rank-selection, and stability procedures. Reconstruction quality was reported as $R^2$ for dense similarity matrices. For SWOW, where positive PPMI entries form a sparse graph and zero entries indicate no positive association above chance, we reported ROC-AUC for discriminating all 476{,}383 upper-triangular positive PPMI pairs from 476{,}383 upper-triangular zero-PPMI pairs sampled without replacement, using reconstructed similarities $WW^\top$ as scores. We used all positive pairs and an equal-size sample of zero pairs because zero-PPMI pairs vastly outnumber positive pairs. This gives a balanced and computationally tractable estimate of the probability that a positive association is ranked above a zero association.

\subsection{THINGS behavioral prediction}

We used crowd-sourced odd-one-out judgments from the THINGS database \citep{Hebart2019}, comprising 4.7 million triplet responses across 1{,}854 objects. Pairwise similarities were constructed from the triplets as described in Section~\ref{sec:similarity_construction}. The SRF rank was selected by cross-validation on the training set, and the final model was evaluated on a held-out test set. To predict triplet judgments, we reconstructed pairwise similarities from the fitted embedding and predicted the odd-one-out as the object not in the most similar pair. Prediction accuracy was measured as the percentage of correctly predicted odd-one-out choices on the held-out test set. We compared SRF with the published 66-dimensional SPoSE embedding \citep{Hebart2020}, which was trained directly on the original triplet judgments. To test whether the recovered pairwise similarities captured behavior beyond the triplet task itself, we additionally correlated SRF-predicted similarities for a separate set of 48 objects with independent human similarity ratings.

\subsection{Word association analysis}

We used the Small World of Words English dataset \citep{DeDeyne2019}, comprising word association data from more than 90{,}000 participants. Each participant saw cue words and provided three associated responses. We used the preprocessed strength file (R123), which aggregates responses across approximately 100 participants per cue. Pairwise similarities were constructed from the association data as described above. To test whether the recovered dimensions generalized beyond the association data themselves, we predicted independent behavioral ratings from the Glasgow Norms \citep{Scott2019} for words present in both datasets. We used ratings including valence, arousal, concreteness, imageability, size, familiarity, and age of acquisition, all collected in separate behavioral experiments. For each rating variable, we fit a ridge-regression model using SRF embeddings as predictors. Predictions were generated with 5-fold nested cross-validation. Within each fold, embedding dimensions were standardized and the regularization strength was selected from $\alpha \in \{0.01, 0.1, 1, 10, 100\}$. Prediction quality was measured as the Spearman correlation between out-of-fold predictions and observed ratings.

\section{Author Contributions}

F.P.M and M.N.H conceived the project. F.P.M. and K.C.L. jointly conceived the method and developed the optimization algorithm, with additional input from F.P. F.P.M. implemented the software. F.P.M. and M.N.H. designed the experiments. F.P.M. performed the experiments. F.P.M. and K.C.L. wrote the original draft. M.N.H. supervised the research, provided critical revisions, and acquired funding. All authors reviewed and approved the final manuscript.

\section{Acknowledgments}

F.P.M. and M.N.H. acknowledge support from a Max Planck Research Group grant of the Max Planck Society awarded to M.N.H. M.N.H. acknowledges support from the ERC Starting Grant COREDIM (ERC-StG-2021-101039712), the Hessian Ministry of Higher Education, Science, Research and Art (LOEWE Start Professorship), and the Deutsche Forschungsgemeinschaft (German Research Foundation, DFG) under Germany’s Excellence Strategy (EXC 3066/1 “The Adaptive Mind”, Project No. 533717223). This study used the high-performance computing capabilities of the Raven and Cobra Linux clusters at the Max Planck Computing \& Data Facility (MPCDF), Garching, Germany (\url{https://www.mpcdf.mpg.de/services/supercomputing/}).
K.C.L. and F.P. were supported by the National Institute of Mental Health Intramural Research Program (ZICMH002968). The contributions of the NIH authors are considered Works of the United States Government. The findings and conclusions presented in this paper are those of the author(s) and do not necessarily reflect the views of the NIH or the U.S. Department of Health and Human Services. The funders had no role in the study design, data collection and analysis, decision to publish, or preparation of the manuscript.

We thank Paolo Papale for sharing the macaque electrophysiology data and Heiko Schütt and the Hebartlab for useful discussions.

\section{Software}

A Python package implementing our method is available at \href{https://github.com/florianmahner/pysrf}{https://github.com/florianmahner/pysrf}. A repository containing all analyses is available at \href{https://github.com/florianmahner/similarity-factorization}{https://github.com/florianmahner/similarity-factorization}.

\clearpage

\bibliography{bibliography}

\begin{thebibliography}{73}
\providecommand{\natexlab}[1]{#1}
\providecommand{\url}[1]{\texttt{#1}}
\expandafter\ifx\csname urlstyle\endcsname\relax
  \providecommand{\doi}[1]{doi: #1}\else
  \providecommand{\doi}{doi: \begingroup \urlstyle{rm}\Url}\fi

\bibitem[Allen et~al.(2022)Allen, St-Yves, Wu, Breedlove, Prince, Dowdle, Nau, Caron, Pestilli, Charest, Hutchinson, Naselaris, and Kay]{Allen2022}
E.~J. Allen, G.~St-Yves, Y.~Wu, J.~L. Breedlove, J.~S. Prince, L.~T. Dowdle, M.~Nau, B.~Caron, F.~Pestilli, I.~Charest, J.~B. Hutchinson, T.~Naselaris, and K.~Kay.
\newblock A massive {7T} {fMRI} dataset to bridge cognitive neuroscience and artificial intelligence.
\newblock \emph{Nat. Neurosci.}, 25:\penalty0 116--126, 2022.

\bibitem[Bao et~al.(2020)Bao, She, McGill, and Tsao]{Bao2020}
P.~Bao, L.~She, M.~McGill, and D.~Y. Tsao.
\newblock A map of object space in primate inferotemporal cortex.
\newblock \emph{Nature}, 583:\penalty0 103--108, 2020.

\bibitem[Bau et~al.(2017)Bau, Zhou, Khosla, Oliva, and Torralba]{Bau2017}
D.~Bau, B.~Zhou, A.~Khosla, A.~Oliva, and A.~Torralba.
\newblock Network dissection: quantifying interpretability of deep visual representations.
\newblock In \emph{Proc. IEEE/CVF Conference on Computer Vision and Pattern Recognition}, pages 3319--3327. IEEE, 2017.

\bibitem[Benjamini and Hochberg(1995)]{Benjamini1995}
Y.~Benjamini and Y.~Hochberg.
\newblock Controlling the false discovery rate: a practical and powerful approach to multiple testing.
\newblock \emph{J. R. Stat. Soc. Series B Stat. Methodol.}, 57\penalty0 (1):\penalty0 289--300, 1995.

\bibitem[Bockes et~al.(2025)Bockes, Hebart, and Lingnau]{Bockes2025}
A.~Bockes, M.~N. Hebart, and A.~Lingnau.
\newblock Revealing key dimensions underlying the recognition of dynamic human actions.
\newblock \emph{Commun. Psychol.}, 3:\penalty0 149, 2025.

\bibitem[Boyd et~al.(2011)Boyd, Parikh, Chu, Peleato, and Eckstein]{Boyd2011}
S.~Boyd, N.~Parikh, E.~Chu, B.~Peleato, and J.~Eckstein.
\newblock Distributed optimization and statistical learning via the alternating direction method of multipliers.
\newblock \emph{Found. Trends Mach. Learn.}, 3:\penalty0 1--122, 2011.

\bibitem[Bracci and Op~de Beeck(2023)]{Bracci2023}
S.~Bracci and H.~P. Op~de Beeck.
\newblock Understanding human object vision: A picture is worth a thousand representations.
\newblock \emph{Annu. Rev. Psychol.}, 74:\penalty0 113--135, 2023.

\bibitem[Bricken et~al.(2023)Bricken, Templeton, Batson, Chen, Jermyn, Conerly, Turner, Anil, Denison, Askell, Lasenby, Wu, Kravec, Schiefer, Maxwell, Joseph, Hatfield-Dodds, Tamkin, Nguyen, McLean, Burke, Hume, Carter, Henighan, and Olah]{Bricken2023}
T.~Bricken, A.~Templeton, J.~Batson, B.~Chen, A.~Jermyn, T.~Conerly, N.~Turner, C.~Anil, C.~Denison, A.~Askell, R.~Lasenby, Y.~Wu, S.~Kravec, N.~Schiefer, T.~Maxwell, N.~Joseph, Z.~Hatfield-Dodds, A.~Tamkin, K.~Nguyen, B.~McLean, J.~E. Burke, T.~Hume, S.~Carter, T.~Henighan, and C.~Olah.
\newblock Towards monosemanticity: Decomposing language models with dictionary learning.
\newblock Transformer Circuits Thread, 2023.

\bibitem[Cand{\`e}s and Recht(2009)]{Candes2009}
E.~J. Cand{\`e}s and B.~Recht.
\newblock Exact matrix completion via convex optimization.
\newblock \emph{Found. Comput. Math.}, 9\penalty0 (6):\penalty0 717--772, 2009.

\bibitem[Church and Hanks(1990)]{Church1990}
K.~W. Church and P.~Hanks.
\newblock Word association norms, mutual information, and lexicography.
\newblock \emph{Comput. Linguist.}, 16\penalty0 (1):\penalty0 22--29, 1990.

\bibitem[Contier et~al.(2024)Contier, Baker, and Hebart]{Contier2024}
O.~Contier, C.~I. Baker, and M.~N. Hebart.
\newblock Distributed representations of behaviour-derived object dimensions in the human visual system.
\newblock \emph{Nat. Hum. Behav.}, 8:\penalty0 2179--2193, 2024.

\bibitem[Costa et~al.(2025)Costa, Fel, Lubana, Tolooshams, and Ba]{Costa2025}
V.~Costa, T.~Fel, E.~S. Lubana, B.~Tolooshams, and D.~Ba.
\newblock From flat to hierarchical: extracting sparse representations with matching pursuit.
\newblock \emph{Adv. Neural Inf. Process. Syst.}, 38, 2025.

\bibitem[De~Deyne et~al.(2019)De~Deyne, Navarro, Perfors, Brysbaert, and Storms]{DeDeyne2019}
S.~De~Deyne, D.~J. Navarro, A.~Perfors, M.~Brysbaert, and G.~Storms.
\newblock The "small world of words" english word association norms for over 12,000 cue words.
\newblock \emph{Behav. Res. Methods}, 51:\penalty0 987--1006, 2019.

\bibitem[DiCarlo et~al.(2012)DiCarlo, Zoccolan, and Rust]{DiCarlo2012}
J.~J. DiCarlo, D.~Zoccolan, and N.~C. Rust.
\newblock How does the brain solve visual object recognition?
\newblock \emph{Neuron}, 73:\penalty0 415--434, 2012.

\bibitem[Ding et~al.(2005)Ding, He, and Simon]{Ding2005}
C.~Ding, X.~He, and H.~D. Simon.
\newblock On the equivalence of nonnegative matrix factorization and spectral clustering.
\newblock In \emph{Proc. SIAM International Conference on Data Mining}, pages 606--610, 2005.

\bibitem[Du et~al.(2025)Du, Fu, Wen, Sun, Peng, Wei, Gao, Wang, Zhang, Li, Qiu, Chang, and He]{Du2025}
C.~Du, K.~Fu, B.~Wen, Y.~Sun, J.~Peng, W.~Wei, Y.~Gao, S.~Wang, C.~Zhang, J.~Li, S.~Qiu, L.~Chang, and H.~He.
\newblock Human-like object concept representations emerge naturally in multimodal large language models.
\newblock \emph{Nat. Mach. Intell.}, 7:\penalty0 860--875, 2025.

\bibitem[Finkelstein et~al.(2002)Finkelstein, Gabrilovich, Matias, Rivlin, Solan, Wolfman, and Ruppin]{Finkelstein2002}
L.~Finkelstein, E.~Gabrilovich, Y.~Matias, E.~Rivlin, Z.~Solan, G.~Wolfman, and E.~Ruppin.
\newblock Placing search in context: The concept revisited.
\newblock \emph{ACM Trans. Inf. Syst.}, 20:\penalty0 116--131, 2002.

\bibitem[Fyshe et~al.(2015)Fyshe, Wehbe, Talukdar, Murphy, and Mitchell]{Fyshe2015}
A.~Fyshe, L.~Wehbe, P.~P. Talukdar, B.~Murphy, and T.~M. Mitchell.
\newblock A compositional and interpretable semantic space.
\newblock In \emph{Proc. Conference of the North American Chapter of the ACL}, pages 32--41. Association for Computational Linguistics, 2015.

\bibitem[G{\"a}rdenfors(2000)]{Gardenfors2000}
P.~G{\"a}rdenfors.
\newblock \emph{Conceptual spaces: the geometry of thought}.
\newblock MIT Press, Cambridge, MA, 2000.

\bibitem[Geirhos et~al.(2021)Geirhos, Narayanappa, Mitzkus, Thieringer, Bethge, Wichmann, and Brendel]{Geirhos2021}
R.~Geirhos, K.~Narayanappa, B.~Mitzkus, T.~Thieringer, M.~Bethge, F.~A. Wichmann, and W.~Brendel.
\newblock Partial success in closing the gap between human and machine vision.
\newblock \emph{Adv. Neural Inf. Process. Syst.}, 34:\penalty0 23885--23899, 2021.

\bibitem[Geirhos et~al.(2024)Geirhos, Zimmermann, Bilodeau, Brendel, and Kim]{Geirhos2024}
R.~Geirhos, R.~S. Zimmermann, B.~Bilodeau, W.~Brendel, and B.~Kim.
\newblock Don't trust your eyes: on the (un)reliability of feature visualizations.
\newblock \emph{Proc. Mach. Learn. Res.}, 235:\penalty0 15294--15330, 2024.

\bibitem[Gross(2011)]{Gross2011}
D.~Gross.
\newblock Recovering low-rank matrices from few coefficients in any basis.
\newblock \emph{IEEE Trans. Inf. Theory}, 57\penalty0 (3):\penalty0 1548--1566, 2011.

\bibitem[Hebart et~al.(2019)Hebart, Dickter, Kidder, Kwok, Corriveau, Van~Wicklin, and Baker]{Hebart2019}
M.~N. Hebart, A.~H. Dickter, A.~Kidder, W.~Y. Kwok, A.~Corriveau, C.~Van~Wicklin, and C.~I. Baker.
\newblock {THINGS}: A database of 1,854 object concepts and more than 26,000 naturalistic object images.
\newblock \emph{PLoS One}, 14:\penalty0 e0223792, 2019.

\bibitem[Hebart et~al.(2020)Hebart, Zheng, Pereira, and Baker]{Hebart2020}
M.~N. Hebart, C.~Y. Zheng, F.~Pereira, and C.~I. Baker.
\newblock Revealing the multidimensional mental representations of natural objects underlying human similarity judgements.
\newblock \emph{Nat. Hum. Behav.}, 4:\penalty0 1173--1185, 2020.

\bibitem[Hebart et~al.(2023)Hebart, Contier, Teichmann, Rockter, Zheng, Kidder, Corriveau, Vaziri-Pashkam, and Baker]{Hebart2023}
M.~N. Hebart, O.~Contier, L.~Teichmann, A.~H. Rockter, C.~Y. Zheng, A.~Kidder, A.~Corriveau, M.~Vaziri-Pashkam, and C.~I. Baker.
\newblock {THINGS}-data, a multimodal collection of large-scale datasets for investigating object representations in human brain and behavior.
\newblock \emph{eLife}, 12:\penalty0 e82580, 2023.

\bibitem[Hill et~al.(2015)Hill, Reichart, and Korhonen]{Hill2015}
F.~Hill, R.~Reichart, and A.~Korhonen.
\newblock {SimLex}-999: Evaluating semantic models with (genuine) similarity estimation.
\newblock \emph{Comput. Linguist.}, 41:\penalty0 665--695, 2015.

\bibitem[Huben et~al.(2024)Huben, Cunningham, Smith, Ewart, and Sharkey]{Cunningham2024}
R.~Huben, H.~Cunningham, L.~R. Smith, A.~Ewart, and L.~Sharkey.
\newblock Sparse autoencoders find highly interpretable features in language models.
\newblock In \emph{Proc. International Conference on Learning Representations}, 2024.

\bibitem[Jagadeesh and Gardner(2022)]{Jagadeesh2022}
A.~V. Jagadeesh and J.~L. Gardner.
\newblock Texture-like representation of objects in human visual cortex.
\newblock \emph{Proc. Natl. Acad. Sci. U. S. A.}, 119:\penalty0 e2115302119, 2022.

\bibitem[Jolliffe(1986)]{Jolliffe1986}
I.~T. Jolliffe.
\newblock Principal components in regression analysis.
\newblock In \emph{Principal Component Analysis}, pages 129--155. Springer, 1986.

\bibitem[Josephs et~al.(2023)Josephs, Hebart, and Konkle]{Josephs2023}
E.~L. Josephs, M.~N. Hebart, and T.~Konkle.
\newblock Dimensions underlying human understanding of the reachable world.
\newblock \emph{Cognition}, 234:\penalty0 105368, 2023.

\bibitem[Khosla et~al.(2022)Khosla, Ratan~Murty, and Kanwisher]{Khosla2022}
M.~Khosla, N.~A. Ratan~Murty, and N.~Kanwisher.
\newblock A highly selective response to food in human visual cortex revealed by hypothesis-free voxel decomposition.
\newblock \emph{Curr. Biol.}, 32\penalty0 (19):\penalty0 4159--4171.e9, 2022.

\bibitem[Khrulkov et~al.(2020)Khrulkov, Mirvakhabova, Ustinova, Oseledets, and Lempitsky]{Khrulkov2020}
V.~Khrulkov, L.~Mirvakhabova, E.~Ustinova, I.~Oseledets, and V.~Lempitsky.
\newblock Hyperbolic image embeddings.
\newblock In \emph{Proc. IEEE/CVF Conference on Computer Vision and Pattern Recognition}, pages 6418--6428. IEEE, 2020.

\bibitem[Klindt et~al.(2025)Klindt, O'Neill, Reizinger, Maurer, and Miolane]{Klindt2025}
D.~Klindt, C.~O'Neill, P.~Reizinger, H.~Maurer, and N.~Miolane.
\newblock From superposition to sparse codes: interpretable representations in neural networks.
\newblock \emph{Preprint at arXiv}, 2503.01824, 2025.

\bibitem[Konkle and Caramazza(2013)]{Konkle2013}
T.~Konkle and A.~Caramazza.
\newblock Tripartite organization of the ventral stream by animacy and object size.
\newblock \emph{J. Neurosci.}, 33\penalty0 (25):\penalty0 10235--10242, 2013.

\bibitem[Kornblith et~al.(2019)Kornblith, Norouzi, Lee, and Hinton]{Kornblith2019}
S.~Kornblith, M.~Norouzi, H.~Lee, and G.~Hinton.
\newblock Similarity of neural network representations revisited.
\newblock \emph{Proc. Mach. Learn. Res.}, 97:\penalty0 3519--3529, 2019.

\bibitem[Kriegeskorte and Kievit(2013)]{Kriegeskorte2013}
N.~Kriegeskorte and R.~A. Kievit.
\newblock Representational geometry: integrating cognition, computation, and the brain.
\newblock \emph{Trends Cogn. Sci.}, 17:\penalty0 401--412, 2013.

\bibitem[Kriegeskorte and Wei(2021)]{Kriegeskorte2021}
N.~Kriegeskorte and X.-X. Wei.
\newblock Neural tuning and representational geometry.
\newblock \emph{Nat. Rev. Neurosci.}, 22\penalty0 (11):\penalty0 703--718, 2021.

\bibitem[Kriegeskorte et~al.(2008)Kriegeskorte, Mur, and Bandettini]{Kriegeskorte2008-01}
N.~Kriegeskorte, M.~Mur, and P.~Bandettini.
\newblock Representational similarity analysis - connecting the branches of systems neuroscience.
\newblock \emph{Front. Syst. Neurosci.}, 2:\penalty0 4, 2008.

\bibitem[Kuang et~al.(2012)Kuang, Ding, and Park]{Kuang2012}
D.~Kuang, C.~Ding, and H.~Park.
\newblock Symmetric nonnegative matrix factorization for graph clustering.
\newblock In \emph{Proc. SIAM International Conference on Data Mining}, pages 106--117. SIAM, 2012.

\bibitem[Lapuschkin et~al.(2019)Lapuschkin, W{\"a}ldchen, Binder, Montavon, Samek, and M{\"u}ller]{Lapuschkin2019}
S.~Lapuschkin, S.~W{\"a}ldchen, A.~Binder, G.~Montavon, W.~Samek, and K.-R. M{\"u}ller.
\newblock Unmasking {Clever Hans} predictors and assessing what machines really learn.
\newblock \emph{Nat. Commun.}, 10:\penalty0 1096, 2019.

\bibitem[Laub and M{\"u}ller(2004)]{Laub2004}
J.~Laub and K.-R. M{\"u}ller.
\newblock Feature discovery in non-metric pairwise data.
\newblock \emph{J. Mach. Learn. Res.}, 5:\penalty0 801--818, 2004.

\bibitem[Laub et~al.(2006)Laub, M{\"u}ller, Wichmann, and Macke]{Laub2006}
J.~Laub, K.-R. M{\"u}ller, F.~A. Wichmann, and J.~H. Macke.
\newblock Inducing metric violations in human similarity judgements.
\newblock \emph{Adv. Neural Inf. Process. Syst.}, 19:\penalty0 777--784, 2006.

\bibitem[Lee and Seung(1999)]{Lee1999}
D.~D. Lee and H.~S. Seung.
\newblock Learning the parts of objects by non-negative matrix factorization.
\newblock \emph{Nature}, 401:\penalty0 788--791, 1999.

\bibitem[Mahner et~al.(2025)Mahner, Muttenthaler, Güçlü, and Hebart]{Mahner2025}
F.~P. Mahner, L.~Muttenthaler, U.~Güçlü, and M.~N. Hebart.
\newblock Dimensions underlying the representational alignment of deep neural networks with humans.
\newblock \emph{Nat. Mach. Intell.}, 7:\penalty0 848--859, 2025.

\bibitem[Mahner et~al.(2026)Mahner, Roth, Lam, Bonner, Pereira, and Hebart]{Mahner2026}
F.~P. Mahner, J.~Roth, K.~C. Lam, M.~F. Bonner, F.~Pereira, and M.~N. Hebart.
\newblock Characterizing universal object representations across vision models.
\newblock \emph{Preprint at arXiv}, 2605.13675, 2026.

\bibitem[Majaj et~al.(2015)Majaj, Hong, Solomon, and DiCarlo]{Majaj2015}
N.~J. Majaj, H.~Hong, E.~A. Solomon, and J.~J. DiCarlo.
\newblock Simple learned weighted sums of inferior temporal neuronal firing rates accurately predict human core object recognition performance.
\newblock \emph{J. Neurosci.}, 35\penalty0 (39):\penalty0 13402--13418, 2015.

\bibitem[Mantel(1967)]{Mantel1967}
N.~Mantel.
\newblock The detection of disease clustering and a generalized regression approach.
\newblock \emph{Cancer Res.}, 27\penalty0 (2):\penalty0 209--220, 1967.

\bibitem[Mur et~al.(2013)Mur, Meys, Bodurka, Goebel, Bandettini, and Kriegeskorte]{Mur2013}
M.~Mur, M.~Meys, J.~Bodurka, R.~Goebel, P.~A. Bandettini, and N.~Kriegeskorte.
\newblock Human object-similarity judgments reflect and transcend the primate-{IT} object representation.
\newblock \emph{Front. Psychol.}, 4:\penalty0 128, 2013.

\bibitem[Muttenthaler et~al.(2025)Muttenthaler, Greff, Born, Spitzer, Kornblith, Mozer, Müller, Unterthiner, and Lampinen]{Muttenthaler2025}
L.~Muttenthaler, K.~Greff, F.~Born, B.~Spitzer, S.~Kornblith, M.~C. Mozer, K.-R. Müller, T.~Unterthiner, and A.~K. Lampinen.
\newblock Aligning machine and human visual representations across abstraction levels.
\newblock \emph{Nature}, 647:\penalty0 349--355, 2025.

\bibitem[Naselaris et~al.(2011)Naselaris, Kay, Nishimoto, and Gallant]{Naselaris2011}
T.~Naselaris, K.~N. Kay, S.~Nishimoto, and J.~L. Gallant.
\newblock Encoding and decoding in {fMRI}.
\newblock \emph{Neuroimage}, 56:\penalty0 400--410, 2011.

\bibitem[Nickel and Kiela(2017)]{Nickel2017}
M.~Nickel and D.~Kiela.
\newblock Poincaré embeddings for learning hierarchical representations.
\newblock \emph{Adv. Neural Inf. Process. Syst.}, 30:\penalty0 6338--6347, 2017.

\bibitem[Nosofsky(1986)]{Nosofsky1986}
R.~M. Nosofsky.
\newblock Attention, similarity, and the identification-categorization relationship.
\newblock \emph{J. Exp. Psychol. Gen.}, 115:\penalty0 39--61, 1986.

\bibitem[Olah et~al.(2017)Olah, Mordvintsev, and Schubert]{Olah2017}
C.~Olah, A.~Mordvintsev, and L.~Schubert.
\newblock Feature visualization.
\newblock \emph{Distill}, 2:\penalty0 e7, 2017.

\bibitem[Papale et~al.(2025)Papale, Wang, Self, and Roelfsema]{Papale2025}
P.~Papale, F.~Wang, M.~W. Self, and P.~R. Roelfsema.
\newblock An extensive dataset of spiking activity to reveal the syntax of the ventral stream.
\newblock \emph{Neuron}, 113:\penalty0 539--553.e5, 2025.

\bibitem[Peterson et~al.(2018)Peterson, Abbott, and Griffiths]{Peterson2018}
J.~C. Peterson, J.~T. Abbott, and T.~L. Griffiths.
\newblock Evaluating (and improving) the correspondence between deep neural networks and human representations.
\newblock \emph{Cogn. Sci.}, 42:\penalty0 2648--2669, 2018.

\bibitem[Radford et~al.(2021)Radford, Kim, Hallacy, Ramesh, Goh, Agarwal, Sastry, Askell, Mishkin, Clark, Krueger, and Sutskever]{Radford2021}
A.~Radford, J.~W. Kim, C.~Hallacy, A.~Ramesh, G.~Goh, S.~Agarwal, G.~Sastry, A.~Askell, P.~Mishkin, J.~Clark, G.~Krueger, and I.~Sutskever.
\newblock Learning transferable visual models from natural language supervision.
\newblock \emph{Proc. Mach. Learn. Res.}, 139:\penalty0 8748--8763, 2021.

\bibitem[Roads and Love(2024)]{Roads2024-01}
B.~D. Roads and B.~C. Love.
\newblock Modeling similarity and psychological space.
\newblock \emph{Annu. Rev. Psychol.}, 75\penalty0 (1):\penalty0 215--240, 2024.

\bibitem[Samek et~al.(2021)Samek, Montavon, Lapuschkin, Anders, and Müller]{Samek2021}
W.~Samek, G.~Montavon, S.~Lapuschkin, C.~J. Anders, and K.-R. Müller.
\newblock Explaining deep neural networks and beyond: a review of methods and applications.
\newblock \emph{Proc. IEEE}, 109\penalty0 (3):\penalty0 247--278, 2021.

\bibitem[Schölkopf and Smola(2002)]{Scholkopf2002}
B.~Schölkopf and A.~J. Smola.
\newblock \emph{Learning with kernels: Support vector machines, regularization, optimization, and beyond}.
\newblock MIT Press, 2002.

\bibitem[Schütt et~al.(2023)Schütt, Kipnis, Diedrichsen, and Kriegeskorte]{Schutt2023}
H.~H. Schütt, A.~D. Kipnis, J.~Diedrichsen, and N.~Kriegeskorte.
\newblock Statistical inference on representational geometries.
\newblock \emph{eLife}, 12:\penalty0 e82566, 2023.

\bibitem[Scott et~al.(2019)Scott, Keitel, Becirspahic, Yao, and Sereno]{Scott2019}
G.~G. Scott, A.~Keitel, M.~Becirspahic, B.~Yao, and S.~C. Sereno.
\newblock The glasgow norms: Ratings of 5,500 words on nine scales.
\newblock \emph{Behav. Res. Methods}, 51:\penalty0 1258--1270, 2019.

\bibitem[Shepard(1957)]{Shepard1957}
R.~N. Shepard.
\newblock Stimulus and response generalization: a stochastic model relating generalization to distance in psychological space.
\newblock \emph{Psychometrika}, 22:\penalty0 325--345, 1957.

\bibitem[Shepard(1962)]{Shepard1962}
R.~N. Shepard.
\newblock The analysis of proximities: Multidimensional scaling with an unknown distance function. {I}.
\newblock \emph{Psychometrika}, 27:\penalty0 125--140, 1962.

\bibitem[Shi et~al.(2017)Shi, Sun, Lu, Hong, and Razaviyayn]{Shi2017}
Q.~Shi, H.~Sun, S.~Lu, M.~Hong, and M.~Razaviyayn.
\newblock Inexact block coordinate descent methods for symmetric nonnegative matrix factorization.
\newblock \emph{IEEE Trans. Signal Process.}, 65:\penalty0 5995--6008, 2017.

\bibitem[Sorscher et~al.(2022)Sorscher, Ganguli, and Sompolinsky]{Sorscher2022}
B.~Sorscher, S.~Ganguli, and H.~Sompolinsky.
\newblock Neural representational geometry underlies few-shot concept learning.
\newblock \emph{Proc. Natl. Acad. Sci. U. S. A.}, 119\penalty0 (43):\penalty0 e2200800119, 2022.

\bibitem[Stoinski et~al.(2024)Stoinski, Perkuhn, and Hebart]{Stoinski2024}
L.~M. Stoinski, J.~Perkuhn, and M.~N. Hebart.
\newblock {THINGSplus}: new norms and metadata for the {THINGS} database of 1854 object concepts and 26,107 natural object images.
\newblock \emph{Behav. Res. Methods}, 56:\penalty0 1583--1603, 2024.

\bibitem[Sucholutsky et~al.(2025)Sucholutsky, Muttenthaler, Weller, Peng, Bobu, Kim, Love, Cueva, Grant, Groen, Achterberg, Tenenbaum, Collins, Hermann, Oktar, Greff, Hebart, Cloos, Kriegeskorte, Jacoby, Zhang, Marjieh, Geirhos, Chen, Kornblith, Rane, Konkle, O'Connell, Unterthiner, Lampinen, Müller, Toneva, and Griffiths]{Sucholutsky2025}
I.~Sucholutsky, L.~Muttenthaler, A.~Weller, A.~Peng, A.~Bobu, B.~Kim, B.~C. Love, C.~J. Cueva, E.~Grant, I.~Groen, J.~Achterberg, J.~B. Tenenbaum, K.~M. Collins, K.~L. Hermann, K.~Oktar, K.~Greff, M.~N. Hebart, N.~Cloos, N.~Kriegeskorte, N.~Jacoby, Q.~Zhang, R.~Marjieh, R.~Geirhos, S.~Chen, S.~Kornblith, S.~Rane, T.~Konkle, T.~P. O'Connell, T.~Unterthiner, A.~K. Lampinen, K.-R. Müller, M.~Toneva, and T.~L. Griffiths.
\newblock Getting aligned on representational alignment.
\newblock \emph{Trans. Mach. Learn. Res.}, 2025.

\bibitem[Teichmann et~al.(2026)Teichmann, Hebart, and Baker]{Teichmann2026}
L.~Teichmann, M.~N. Hebart, and C.~I. Baker.
\newblock Dynamic representation of multidimensional object properties in the human brain.
\newblock \emph{J. Neurosci.}, 46:\penalty0 e1057252026, 2026.

\bibitem[Thasarathan et~al.(2025)Thasarathan, Forsyth, Fel, Kowal, and Derpanis]{Thasarathan2025}
H.~Thasarathan, J.~Forsyth, T.~Fel, M.~Kowal, and K.~G. Derpanis.
\newblock Universal sparse autoencoders: Interpretable cross-model concept alignment.
\newblock \emph{Proc. Mach. Learn. Res.}, 267:\penalty0 59304--59325, 2025.

\bibitem[Torgerson(1952)]{Torgerson1952}
W.~S. Torgerson.
\newblock Multidimensional scaling: {I}. theory and method.
\newblock \emph{Psychometrika}, 17:\penalty0 401--419, 1952.

\bibitem[van Bree and Hebart(2026)]{vanBree2026}
S.~van Bree and M.~N. Hebart.
\newblock Shared and distinct object spaces in human and macaque inferotemporal cortex.
\newblock \emph{Preprint at bioRxiv}, 2026.05.20.724014, 2026.

\bibitem[van~der Maaten and Hinton(2008)]{Maaten2008}
L.~van~der Maaten and G.~Hinton.
\newblock Visualizing data using t-{SNE}.
\newblock \emph{J. Mach. Learn. Res.}, 9:\penalty0 2579--2605, 2008.

\bibitem[Zeng and Gallant(2025)]{Zeng2025}
A.~Zeng and J.~L. Gallant.
\newblock Disentangling superpositions: interpretable brain encoding model with sparse concept atoms.
\newblock \emph{Preprint at bioRxiv}, 2025.11.29.691321, 2025.

\end{thebibliography}
\bibliographystyle{plainnat}

\clearpage

\appendix
\setcounter{figure}{0}
\renewcommand{\thefigure}{S\arabic{figure}}

\section{Uncovering behavioral dimensions from trial-based data}

\begin{figure}[h]
  \centering
  \includegraphics[width=1.0\linewidth]{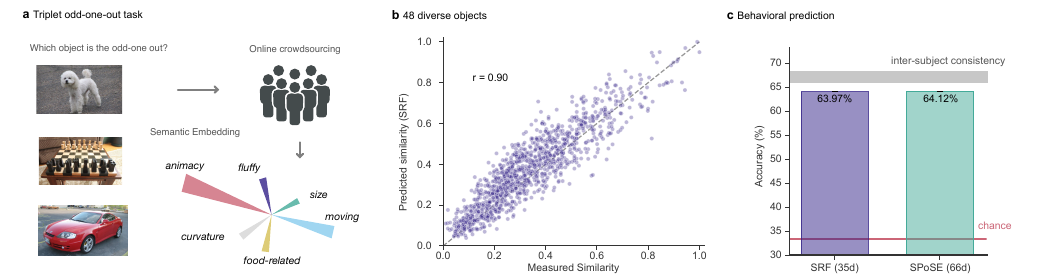}
  \caption{\textbf{SRF recovers behavioral similarity structure from sparse triplet data.}
  \textbf{a},~In the THINGS odd-one-out task \citep{Hebart2019}, participants view three object images and select the least similar one. These triplet judgments are aggregated into pairwise similarities with Laplace smoothing and factorized by SRF into interpretable dimensions.
  \textbf{b},~Predicted pairwise similarities from SRF correlate strongly with independent human similarity ratings collected for a separate set of 48 objects.
  \textbf{c},~Held-out triplet prediction accuracy. SRF at the cross-validated rank approaches the published SPoSE embedding while using roughly half as many dimensions, and reaches 89.23\% of the chance-corrected behavioral noise ceiling.}
  \label{fig:things}
\end{figure}

Here we ask whether SRF can match a state-of-the-art neural-network approach for discovering interpretable dimensions from large-scale behavioral judgments. We answered this on the THINGS odd-one-out dataset \citep{Hebart2023}, which contains 4.7 million triplet odd-one-out judgments across 1{,}854 object concepts. In each trial, participants selected the image that least fit with two others (Fig.~\ref{fig:things}a). We compare SRF to the neural network model SPoSE \citep{Hebart2020}, a very influential model that learns interpretable dimensions directly from triplet trials and has been used to study brain organization \citep{Contier2024,Teichmann2026} and extended to other domains \citep{Josephs2023,Bockes2025}. Our model, in contrast, receives only the pairwise similarity matrix aggregated from the trial-level responses (see Methods).

Cross-validation selected 35 dimensions, compared with 66 for the published SPoSE embedding \citep{Hebart2023}. While the dimensionality was notably lower, it remained open to what degree this more compact representation preserved the ability to predict independent behavioral data. To evaluate the predictive performance of this embedding, we carried out two tests. First, we reconstructed pairwise similarities from the learned similarity embedding and correlated them with independent human similarity ratings collected for a separate set of 48 objects not used during estimation \citep{Hebart2020}. The correlation was high ($r = 0.90$, Fig.~\ref{fig:things}b), comparable to the result reported in the original SPoSE paper ($r = 0.90$) while using substantially fewer dimensions. Second, we tested whether the embedding could predict held-out triplet odd-one-out judgments. The embedding achieved $63.97\% \pm 0.02\%$ accuracy, approaching SPoSE ($64.12\%$, Fig.~\ref{fig:things}c). The inter-subject consistency, defined as the consistency with which humans agree on repeated triplets, was $67.67\% \pm 1.08\%$. Relative to the three-choice chance level of $33.33\%$, the model reached $89.23\%$ of this behavioral ceiling.

Together, these results demonstrate that our model recovers behaviorally relevant dimensions on par with a neural network optimized on a task-specific triplet objective. This broadens the kinds of behavioral data from which such dimensions can be estimated to settings where researchers collect only pairwise or partially sampled similarity judgments.

\section{Optimization and convergence of SRF}\label{sec:srf-optim}

\label{supp:math_derivation}

In the following, we provide the formal derivation of the Similarity-Based Representation Factorization (SRF) algorithm. Throughout the following discussion, we assume $S \in \mathbb{R}^{n \times n}_{\geq 0}$ is a non-negative, symmetric similarity matrix (for example an RBF-kernel Gram matrix). Denote $\mathcal{W} := \left\{ W \in \mathbb{R}^{n \times r} : W \geq 0 \right\}$. We consider representing $S$ with a symmetric, low-rank factorization of the form:

\begin{equation}
    S \approx WW^{\top}, \quad WW^{\top} \in [ \min(S), \max(S)].
\end{equation}

In particular, $W$ is constrained only by non-negativity and by the bound on the reconstruction. In addition, $S$ may be only partially observed, with some entries $S_{ij}$ missing. Here we assume that missingness is completely at random. Missing entries are handled using a symmetric binary mask $\mathcal{M} \in \{ 0, 1\}^{n\times n}$, where $\mathcal{M}_{ij}=\mathcal{M}_{ji}=1$ if $S_{ij}$ is observed, and $0$ otherwise. We assume the diagonal is observed, so $\mathcal{M}_{ii}=1$. This gives the following least-squares problem

\begin{equation}
    \min_{W \in \mathcal{W}} \frac{1}{2} \left\Vert \mathcal{M} \odot \left(S - WW^{\top}  \right) \right\Vert^2_F, \quad \text{such that} \ \min(S) \leq [WW^{\top}]_{ab} \leq \max(S) \ \ \forall (a,b)
\end{equation}

The constrained factorization can be optimized via the Alternating Direction Method of Multipliers (ADMM) framework \citep{Boyd2011}. The augmented Lagrangian is

\begin{equation}
    \mathcal{L}_{\rho}(W,Z,\alpha_Z) = \tfrac{1}{2}\|\mathcal{M}\odot( S - Z )\|_F^2
+ \mathcal{I}_{\mathcal{W}}(W)
+ \langle\alpha_Z,\,Z - W W^\top\rangle
+ \tfrac{\rho}{2}\|Z - W W^\top\|_F^2 + \mathcal{I}_{\mathcal{Z}}(Z).
\end{equation}

The auxiliary $Z$ inherits the box $\mathcal{Z} = \left\{ Z \in \mathbb{R}^{n \times n}_{\geq 0} : Z_{ab} \in [ \min(S), \max(S)] \right\}$ from the data range, and $\mathcal{I}_{\mathcal{S}}$ denotes the convex indicator of the set $\mathcal{S}$. The block updates are

\begin{align}
W^{(i+1)} &= \underset{W \in \mathcal{W}}{\arg\min}\
\frac{\rho}{2}\bigl\| Z^{(i)} - W W^{\top} + \tfrac{1}{\rho}\alpha_Z^{(i)}\bigr\|_F^2 \label{eqt:sym_sub_W} \\
Z^{(i+1)} &= \underset{Z \in \mathcal{Z}}{\arg\min}\
\tfrac{1}{2}\|\mathcal{M}\odot(S-Z)\|_F^2
+ \tfrac{\rho}{2}\bigl\| Z - W^{(i+1)}(W^{(i+1)})^{\top} + \tfrac{1}{\rho}\alpha_Z^{(i)}\bigr\|_F^2 \label{eqt:sym_sub_Z} \\
\alpha_Z^{(i+1)} &\leftarrow \alpha_Z^{(i)} + \rho\bigl( Z^{(i+1)} - W^{(i+1)}(W^{(i+1)})^{\top}\bigr). \nonumber
\end{align}

Throughout we use the shorthand
\begin{equation}
R^{(i)} := W^{(i)}(W^{(i)})^{\top},\qquad
\widetilde Z^{(i)} := Z^{(i)} + \rho^{-1}\alpha_Z^{(i)},\qquad
h_i(W) := \tfrac{\rho}{2}\bigl\|\widetilde Z^{(i)} - W W^{\top}\bigr\|_F^2.
\end{equation}

The $Z$-update has the entry-wise closed form
\begin{equation}
Z^{(i+1)}_{ab}
= \mathrm{Proj}_{\mathcal{Z}}\!\Biggl(
  \frac{\mathcal{M}_{ab}\,S_{ab} + \rho\,R^{(i+1)}_{ab} - \alpha_{Z,ab}^{(i)}}{\rho + \mathcal{M}_{ab}}
\Biggr).
\label{eqt:sym_Z_closedform}
\end{equation}

The $W$-subproblem is non-convex because the objective is quartic in the entries of $W$. We therefore solve it inexactly by the sBSUM scheme of \citet{Shi2017} (see Supplementary Information~\ref{sec:bsum}).

\subsection{Convergence and the lower bound on \texorpdfstring{$\rho$}{rho}}

We now prove a non-increasing Lagrangian property for the ADMM scheme. The dual-variable difference can be controlled tightly through the Karush--Kuhn--Tucker (KKT) condition of the $Z$-update.


\begin{lemma}[KKT identity]
\label{lem:sym_KKT_alpha}
Suppose at iteration $i$ every entry $(a,b)$ of $Z^{(i+1)}$ is the
\emph{interior} solution of \eqref{eqt:sym_sub_Z} (the projection in \eqref{eqt:sym_Z_closedform} is inactive). Then
\begin{equation}
\alpha_Z^{(i+1)} \;=\; \mathcal{M} \odot (S - Z^{(i+1)}).
\label{eqt:sym_KKT_alpha}
\end{equation}
\end{lemma}

\begin{proof}
For each entry $(a,b)$, the unconstrained first-order condition for $Z_{ab}^{(i+1)}$ is $\mathcal{M}_{ab}(Z_{ab}^{(i+1)} - S_{ab}) + \rho(Z_{ab}^{(i+1)} - R_{ab}^{(i+1)}) + \alpha_{Z,ab}^{(i)} = 0$. Combining with the dual update $\alpha_{Z,ab}^{(i+1)} = \alpha_{Z,ab}^{(i)} + \rho(Z_{ab}^{(i+1)} - R_{ab}^{(i+1)})$ gives $\alpha_{Z,ab}^{(i+1)} = \mathcal{M}_{ab} \cdot (S_{ab} - Z_{ab}^{(i+1)})$.
\end{proof}

Applying \eqref{eqt:sym_KKT_alpha} at iterations $(i)$ and $(i+1)$ gives

\begin{equation}
\alpha_Z^{(i+1)} - \alpha_Z^{(i)}
\;=\; \mathcal{M} \odot (Z^{(i)} - Z^{(i+1)})
\;\;\Longrightarrow\;\;
\|\alpha_Z^{(i+1)} - \alpha_Z^{(i)}\|_F^2 \;\leq\; \|\Delta Z^{(i)}\|_F^2,
\label{eqt:sym_alpha_tight}
\end{equation}

Thus, if the projection is inactive at iterations $i$ and $i+1$ the bound is independent of $\Delta R^{(i)} = \Delta(W W^\top)^{(i)}$. In this case, the change in the dual variable $\alpha$ is controlled entirely by the change in $\Delta Z^{(i)}$ on the observed entries. When the projection is active, the bound on $\|\alpha_Z^{(i+1)} - \alpha_Z^{(i)}\|_F^2$ includes an additional term depending on $\Vert W^{(i)}\Vert_{\mathrm{op}}$, which is finite. We address this below.

\begin{theorem}[Non-increasing Lagrangian]
\label{thm:sym_decreasing_tight}
Suppose that
\begin{enumerate}
\item[(i)] the $W$-inner solver of \eqref{eqt:sym_sub_W} satisfies the \emph{descent property}
\begin{equation}
h_i\bigl(W^{(i+1)}\bigr) \;\leq\; h_i\bigl(W^{(i)}\bigr);
\label{eqt:descent_only}
\end{equation}
\item[(ii)] at every iteration, the projection onto $\mathcal{Z}$ in the $Z$-update \eqref{eqt:sym_Z_closedform} is inactive at every entry, so that Lemma~\ref{lem:sym_KKT_alpha} applies.
\end{enumerate}
Then for any $\rho \geq \sqrt{2}$ and every iteration $i$,
\begin{equation}
\mathcal{L}_{\rho}\!\bigl(W^{(i+1)},Z^{(i+1)},\alpha_Z^{(i+1)}\bigr)
\;\le\; \mathcal{L}_{\rho}\!\bigl(W^{(i)},Z^{(i)},\alpha_Z^{(i)}\bigr).
\label{eqt:sym_descent_thm}
\end{equation}
\end{theorem}
\begin{proof}
Adding and subtracting intermediate terms, we decompose the change in $\mathcal{L}_\rho$ as
\begin{align*}
\Delta \mathcal{L}_\rho^{(i)}
&\;=\; \underbrace{\frac{1}{\rho}\|\alpha_Z^{(i+1)}-\alpha_Z^{(i)}\|_F^2}_{(I)}
\;+\; \underbrace{\bigl[\mathcal{L}_\rho(W^{(i+1)},Z^{(i+1)},\alpha_Z^{(i)}) - \mathcal{L}_\rho(W^{(i+1)},Z^{(i)},\alpha_Z^{(i)})\bigr]}_{(\mathcal{A})} \\
&\;+\; \underbrace{\bigl[\mathcal{L}_\rho(W^{(i+1)},Z^{(i)},\alpha_Z^{(i)}) - \mathcal{L}_\rho(W^{(i)},Z^{(i)},\alpha_Z^{(i)})\bigr]}_{(\mathcal{B})}.
\end{align*}
The $Z$-block term $(\mathcal{A})$ is bounded by strong convexity of the $Z$-subproblem in $Z$:
\begin{equation}
(\mathcal{A}) \;\leq\; -\tfrac{\rho}{2}\,\|\Delta Z^{(i)}\|_F^2.
\label{eqt:sym_A_bound}
\end{equation}
The $W$-block term $(\mathcal{B})$ is exactly the change in the $W$-subproblem objective: 
\[
(\mathcal{B}) = h_i(W^{(i+1)}) - h_i(W^{(i)}) \leq 0
\]

by hypothesis (i). Term $(I)$ is bounded by \eqref{eqt:sym_alpha_tight}. Combining,

\begin{align}
\Delta\mathcal{L}_\rho^{(i)}
\;\leq\; \tfrac{1}{\rho}\|\Delta Z^{(i)}\|_F^2 - \tfrac{\rho}{2}\|\Delta Z^{(i)}\|_F^2 + 0
\;=\; \Bigl(\tfrac{1}{\rho} - \tfrac{\rho}{2}\Bigr)\|\Delta Z^{(i)}\|_F^2
\;\leq\; 0
\quad\text{whenever}\quad \rho \geq \sqrt{2}.
\end{align}
\end{proof}

The threshold $\rho \geq \sqrt{2}$ is independent of $\|W\|_{\mathrm{op}}$, $n$, $r$, or any other dimension constant.

\paragraph{Remark on hypothesis (i).}
Hypothesis (i) only requires descent of the $W$-subproblem objective. Hence, any inner solver that does not increase $h_i$ is admissible. This includes majorization-minimization schemes such as sBSUM, and does not require the majorizer to be strongly convex. In the next appendix section, we discuss why this descent condition is sufficient for ADMM to tolerate inexact non-convex inner solves.

\paragraph{Remark on hypothesis (ii).}
The projection onto $\mathcal{Z}$ in the $Z$-update is active on entry $(a,b)$ if and only if the unprojected value

\[
\widehat Z_{ab}^{(i+1)} := \bigl[\mathcal{M}_{ab} \cdot S_{ab} + \rho R_{ab}^{(i+1)} - \alpha_{Z,ab}^{(i)}\bigr] \big/ (\rho + \mathcal{M}_{ab})
\]

falls outside the reconstruction bound. We argue that, for normalized symmetric similarity matrices, this case is unlikely to occur. For this discussion, assume that $S$ is bounded and scaled to $[0, 1]^{n \times n}$:

\begin{enumerate}
    \item \textbf{Observed entries ($\mathcal{M}_{ab}=1$).} 
    By induction, $\alpha_{Z,ab}^{(i)} = S_{ab} - Z_{ab}^{(i)} \in [-1, 1]$. The unprojected value $\widehat Z_{ab}^{(i+1)}$ is therefore a $\rho$-weighted average of $S_{ab}$ and $R_{ab}^{(i+1)}$, up to the current dual correction. The projection can become active only if $(WW^{\top})_{ab}$ leaves the reconstruction bounds.
    \item \textbf{Missing entries ($\mathcal{M}_{ab}=0$).} 
    The unprojected value reduces to $R^{(i+1)}_{ab} - \rho^{-1}\alpha_{Z,ab}^{(i)}$. 
    By induction, $\alpha_{Z,ab}^{(i)} = 0$ on missing entries, so
    \[
    \widehat Z_{ab}^{(i+1)} = (WW^{\top})_{ab}.
    \]
    By the Cauchy--Schwarz inequality, this is bounded by $\sqrt{(WW^{\top})_{aa}(WW^{\top})_{bb}} \approx 1$, since the diagonal of a scaled similarity matrix is anchored to $S_{ii}=1$ (always observed).
\end{enumerate}

Therefore, $\Vert W^{(i)} \Vert_{\mathrm{op}}$ is finite at every iteration $i$. If the projection is active on a small set of entries $\mathcal{S}_i \subset \{1,\dots,n\}^2$, the bound in \eqref{eqt:sym_alpha_tight} on $\|\alpha^{(i+1)} - \alpha^{(i)}\|_F^2$ includes an additional term involving the constraint residual on $\mathcal{S}_i$. Let $\sigma_W^\star := \sup_i \|W^{(i)}\|_{\mathrm{op}}$, which is finite by the boundedness argument above. The resulting threshold is 

\[\rho \geq \max(\sqrt{2},\ \sqrt{8|\mathcal{S}_i|(\sigma_W^\star)^2/n^2}).
\]
When $|\mathcal{S}_i|/n^2 \ll 1$ the second condition is dominated by the first, and Theorem~\ref{thm:sym_decreasing_tight} still holds with $\rho \geq \sqrt{2}$.

Thus, the descent guarantee holds without requiring an excessively large penalty parameter $\rho$. This is useful in practice, since large values of $\rho$ can make the augmented penalty dominate the updates, reducing adaptivity and potentially leading to poorer stationary solutions. In all experiments, we set $\rho = 3$, which is above the sufficient threshold $\sqrt{2}$.

\subsection{Update scheme for \texorpdfstring{$W$}{W}-subproblem}\label{sec:bsum}

The $W$-subproblem 
\[
\min_{W \geq 0} h_i(W), \quad h_i(W) = \frac{\rho}{2} \Vert \widetilde{Z}^{(i)} - WW^\top \Vert^2_F
\]
is non-convex and quartic in the entries of $W$. We adopt the scalar BSUM scheme as the inner solver and describe the algorithm below in our notation.
When all entries of $W$ except $W_{ir}$ are fixed, the inner cost reduces to a univariate quartic on $W \in [0, \infty)$. Let $W$ denote the current intra-sweep iterate, set $\widetilde{Z} := \widetilde{Z}^{(i)}$, and write the candidate as $w \geq 0$ with displacement $\delta := w - W_{ir}$. By direct expansion, 
\begin{equation}
h_i(W + \delta E_{ir}) - h_i(W) = \frac{\rho}{2}g_{ir}(\delta), \quad g_{ir}(\delta) := \frac{a}{4} \delta^4 + \frac{b}{3} \delta^3 + \frac{c}{2} \delta^2 + d \delta
\end{equation}
with coefficients
\begin{align}
a &= 4 \\
b &= 12 W_{ir} \\
c &= 4 \left[ (WW^{\top})_{ii} - \widetilde{Z}_{ii} + (W^{\top} W)_{rr} + W^2_{ir} \right] \\
d &= 4 \left[ (WW^{\top} - \widetilde{Z}) W\right]_{ir}
\end{align}

The polynomial $g_{ir}$ is globally convex whenever $c \geq b^2/(3a)$. When this condition fails, \citet{Shi2017} add the smallest quadratic correction that makes the scalar majorizer convex
\begin{equation}
\widetilde{c} := \max\left( \frac{b^2}{3a} - c, 0\right), \quad \widetilde{g}_{ir}(\delta) := g_{ir}(\delta) + \frac{1}{2} \widetilde{c}\delta^2
\end{equation}
This correction step satisfies 
$\widetilde{g}_{ir}(0) = g_{ir}(0) = 0$, $\widetilde{g}_{ir}(\delta) \geq g_{ir}(\delta)$ for all $\delta \geq -W_{ir}$, i.e. for all $w \geq 0$. 
Moreover, $\widetilde{g}''_{ir}(\delta) \geq 0$ with strict convexity when $c > \frac{b^2}{3a}$. This correction step is applied only when needed and is the smallest quadratic perturbation that ensures convexity. Following \citet{Shi2017}, use $W_{ir}=b/(3a)$ to rewrite the first-order condition for the majorizer as a cubic equation in the candidate entry value. Define
\begin{equation}
p := \frac{3ac-b^2}{3a^2}, \qquad
q := \frac{9abc - 27a^2d - 2b^3}{27a^3}, \qquad
\Delta := \frac{q^2}{4} + \frac{p^3}{27}.
\end{equation}
The scalar sBSUM update is
\begin{equation}
W_{ir}^{\mathrm{new}} = [w_\star]_+,
\qquad
w_\star =
\begin{cases}
\sqrt[3]{\,q/2 - \sqrt{\Delta}\,}
+
\sqrt[3]{\,q/2 + \sqrt{\Delta}\,},
& \text{if } c > b^2/(3a),\\[6pt]
\sqrt[3]{\,b^3/(27a^3) - d/a\,},
& \text{if } c \leq b^2/(3a),
\end{cases}
\label{eqt:sbsum_closed_form}
\end{equation}
where cube roots denote real cube roots. For each entry-update, $\widetilde{g}_{ir}$ is a tight upper bound of $g_{ir}$ at the current value. The closed-form sBSUM update returns the unique one-dimensional minimum over $w \geq 0$. Consequently,
\[
h_i(W^{\text{new}}) \leq h_i(W)
\]
after every entry update. Cycling through all $nk$ entries gives the full-sweep descent
\[
h_i(W^{(i+1)}) \leq h_i(W^{(i)})
\]
which is exactly hypothesis (i) of Theorem~\ref{thm:sym_decreasing_tight}.

\subsection{Algorithm summary}

\begin{algorithm}[!ht]
\caption{Similarity-Based Representation Factorization (SRF)}
\small
\begin{algorithmic}[1]
\Require
    \State Similarity matrix $S \in \mathbb{R}^{n \times n}$
    \State Observation mask $M \in \{0,1\}^{n \times n}$
    \State Target rank $r$, penalty $\rho > 0$
\Ensure Non-negative embedding $W \in \mathbb{R}^{n \times r}_{\ge 0}$

\State \textbf{Initialize:} $W \sim \mathcal{U}(0,1)$, $Z \leftarrow W W^\top$, $\Lambda \leftarrow \mathbf{0}_{n \times n}$

\While{not converged}
    \Statex \textbf{1. W-update}
    \State $T \leftarrow Z + \rho^{-1}\Lambda$ \Comment{Construct proxy target}
    \State $W \leftarrow \arg\min_{W \ge 0} \|T - W W^\top\|_F^2$ \Comment{Solve by scalar sBSUM}

    \Statex \textbf{2. Z-update}
    \State $R \leftarrow W W^\top$
    \For{all entries $(i,j)$}
        \If{$M_{ij} = 1$}
            \State $\widehat Z_{ij} \leftarrow (S_{ij} + \rho R_{ij} - \Lambda_{ij}) / (1+\rho)$
        \Else
            \State $\widehat Z_{ij} \leftarrow R_{ij} - \Lambda_{ij}/\rho$
        \EndIf
        \State $Z_{ij} \leftarrow \operatorname{Proj}_{[\min(S),\,\max(S)]}(\widehat Z_{ij})$
    \EndFor
    \State $Z \leftarrow \frac{1}{2}(Z + Z^\top)$ \Comment{Enforce symmetry}

    \Statex \textbf{3. Dual update}
    \State $\Lambda \leftarrow \Lambda + \rho(Z - W W^\top)$

    \State \textbf{Check Convergence:} $\|Z - WW^\top\|_F < \varepsilon$
\EndWhile
\State \Return $W$
\end{algorithmic}
\end{algorithm}

\section{Cross-validation procedure of SRF}\label{sec:cross-validation}


\subsection{Identifiability when \texorpdfstring{$r$}{r} is misspecified}
\label{subsec:misspecify_k}

Throughout this subsection, bold symbols (e.g., $\rmW$) denote random matrices. Expectations are taken with respect to their joint distribution. We now establish a lower bound on the relative factor error 
\[
\frac{
\mathbb{E}\bigl[\|\rmW^* - \rmW_2\rmP\|_F^2\bigr]
}{
\mathbb{E}\bigl[\|\rmW^*\|_F^2\bigr]
}
\]
incurred when fitting symmetric NMF with a misspecified latent dimension $r_2 \neq r^*$. Here $\rmW^*$ denotes the random factor matrix at the well-specified rank $r^*$, $\rmW_2$ the random factor matrix at the misspecified rank $r_2$, and $\rmP$ the column-matching permutation that minimizes $\Vert\rmW^* - \rmW_2 \rmP \Vert_F^2$ in each realization (with zero-column padding when $r_2 > r^*$). The permutation is required because NMF solutions form an equivalence class under column permutation.

Let $\rmS \in \mathbb{R}^{n \times n}$ be a fixed target matrix and let $\rmW_i \in \mathbb{R}_{\geq 0}^{n \times r_i}$ for $i \in \{1, 2\}$ be random non-negative factor matrices with bounded entries and latent dimensions $r_1, r_2$, in which both approximate the same data matrix $\rmS \in \mathbb{R}^{n \times n}$ in expectation. In particular, we assume the following conditions on the distribution of $(\rmW_1, \rmW_2)$
\begin{align}
\mathbb{E}\bigl[\|\rmS - \rmW_1\rmW_1^\top\|_F^2\bigr]
&=
\mathbb{E}\bigl[\|\rmS - \rmW_2\rmW_2^\top\|_F^2\bigr],
\tag{A1}
\label{ass:equal_error}
\\
\mathbb{E}[\rmW_1\rmW_1^\top]
&=
\mathbb{E}[\rmW_2\rmW_2^\top].
\tag{A2}
\label{ass:equal_reconstruction}
\end{align}

We also assume that all columns of $\rmW_i$ have the same expected squared norm,
\begin{equation}
\mathbb{E}\biggl[
\sum_{j=1}^n (\rmW_i)_{j\kappa}^2
\biggr]
:= \sigma_{\rmW_i}^2
\qquad
\text{for every column index } \kappa
\tag{A3}
\label{ass:column_norm}
\end{equation}
together with a structural independence assumption:
\begin{equation}
\text{the rows of } \rmW_i \text{ are mutually independent for each } i \in \{1, 2\}. \tag{A4} \label{ass:row_independence}
\end{equation}
We also assume weak column overlap:
\begin{equation}
\mathbb{E}\bigl[
(\rmW_i^\top \rmW_i)_{\kappa\lambda}^2
\bigr]
=
\mathcal{O}(n)
\qquad
\text{for all } \kappa \neq \lambda .
\tag{A5}
\label{ass:weak_column_overlap}
\end{equation}
This condition says that distinct latent dimensions do not share enough mass across items for their cross-column inner products to contribute at the same leading order as the squared column norms. Assumptions \ref{ass:equal_error} and \ref{ass:equal_reconstruction} describe the case where the two fits are statistically indistinguishable by reconstruction loss, since their expected residuals and expected reconstructions agree. In this case, the validation error cannot distinguish $r_2$ from $r_1$, so the remaining question is whether the factor itself is identifiable. Assumption \ref{ass:column_norm} imposes equality of expected squared column norms. Assumptions \ref{ass:row_independence} and \ref{ass:weak_column_overlap} control the off-diagonal moments of $\rmW_i^\top \rmW_i$ and justify the variance estimates used in the derivation below.

We first expand \ref{ass:equal_error} and reduce it to a relation involving only the expected squared column norms $\sigma_{\rmW_i}^2$ and dimensions $r_i$. For each $i \in \{1, 2\}$,

\begin{equation}
\mathbb{E}\bigl[\|\rmS - \rmW_i \rmW_i^\top\|_F^2\bigr]
= \mathbb{E}\bigl[\|\rmS\|_F^2\bigr]
\;-\; 2\,\mathbb{E}\bigl[\langle \rmS, \rmW_i \rmW_i^\top \rangle\bigr]
\;+\; \mathbb{E}\bigl[\|\rmW_i \rmW_i^\top\|_F^2\bigr],
\label{eqt:sym_loss_expansion}
\end{equation}

where $\langle A, B \rangle := \mathrm{Tr}(A^\top B)$ denotes the Frobenius inner product. The cross term in \eqref{eqt:sym_loss_expansion} can be written as

\begin{equation}
\mathbb{E}\bigl[\langle \rmS, \rmW_i \rmW_i^\top \rangle\bigr]
\;=\; \bigl\langle \mathbb{E}[\rmS],\, \mathbb{E}[\rmW_i \rmW_i^\top] \bigr\rangle,
\end{equation}

provided that $\rmS$ is fixed, or more generally that $\rmS$ and $\rmW_i$ are independent under the modeling distribution. By \ref{ass:equal_reconstruction}, $\mathbb{E}[\rmW_1 \rmW_1^\top] = \mathbb{E}[\rmW_2 \rmW_2^\top]$, so the cross term takes the same value for $i = 1$ and $i = 2$ and cancels when we equate the two sides of \ref{ass:equal_error}. The term $\mathbb{E}[\|\rmS\|_F^2]$ also cancels, leaving
\begin{equation}
\mathbb{E}\bigl[\|\rmW_1 \rmW_1^\top\|_F^2\bigr]
\;=\; \mathbb{E}\bigl[\|\rmW_2 \rmW_2^\top\|_F^2\bigr].
\end{equation}
Using the cyclic identity
$\|\rmW \rmW^\top\|_F^2 = \mathrm{Tr}((\rmW \rmW^\top)^\top (\rmW \rmW^\top))
= \mathrm{Tr}(\rmW \rmW^\top \rmW \rmW^\top) = \mathrm{Tr}((\rmW^\top \rmW)^2)$, we have
\begin{equation}
\mathbb{E}\bigl[\mathrm{Tr}\bigl((\rmW_1^\top \rmW_1)^2\bigr)\bigr]
\;=\; \mathbb{E}\bigl[\mathrm{Tr}\bigl((\rmW_2^\top \rmW_2)^2\bigr)\bigr].
\label{eqt:sym_var_eq}
\end{equation}
For a single $\rmW_i$ of dimension $n \times r_i$, by definition,
\begin{align}
\mathbb{E}\bigl[\mathrm{Tr}\bigl((\rmW_i^\top \rmW_i)^2\bigr)\bigr]
&\;=\; \mathbb{E}\bigl[\|\rmW_i^\top \rmW_i\|_F^2\bigr]
\;=\; \sum_{\kappa, \lambda = 1}^{r_i} \mathbb{E}\bigl[(\rmW_i^\top \rmW_i)_{\kappa \lambda}^2\bigr] \\
&\;=\; r_i\, \sigma_{\rmW_i}^4 \;+\; \mathcal{O}(r_i^2\, n)
\;=\; r_i\, \sigma_{\rmW_i}^4\,(1 + o(1)),
\end{align}

Since $\sigma_{\rmW_i}^4 = \Theta(n^2)$, the leading term $r_i\sigma_{\rmW_i}^4$ dominates the correction $\mathcal{O}(r_i^2 n)$ for fixed $r_i$, or more generally when $r_i=o(n)$.\footnote{where $f(n) = \Theta(g(n))$ means there exist constants $c_1$, $c_2 > 0$ such that $c_1 g(n) \leq f(n) \leq c_2 g(n)$ for any sufficiently large $n$.} Substituting into \eqref{eqt:sym_var_eq} and retaining the leading-order scaling,
\begin{equation}
  r_1\,\sigma_{\rmW_1}^4
  =
  r_2\,\sigma_{\rmW_2}^4\,(1+o(1))
  \;\;\Longrightarrow\;\;
  \frac{\sigma_{\rmW_1}^2}{\sigma_{\rmW_2}^2}
  =
  \sqrt{\frac{r_2}{r_1}}\,(1+o(1)).
\label{eqt:sym_sigma_ratio}
\end{equation}
Moreover, since the Frobenius norm of $\rmW_i$ is the sum of squared entries, we can group $W$ by column and apply assumption \eqref{ass:column_norm} to get
\begin{equation}
\mathbb{E}\bigl[\|\rmW_i\|_F^2\bigr]
\;=\; \sum_{\kappa = 1}^{r_i} \mathbb{E}\bigl[\textstyle\sum_{j=1}^n (\rmW_i)_{j\kappa}^2\bigr]
\;=\; r_i\, \sigma_{\rmW_i}^2.
\end{equation}
Substituting \eqref{eqt:sym_sigma_ratio}:
  \begin{equation}
  \frac{\mathbb{E}\|\rmW_2\|_F^2}{\mathbb{E}\|\rmW_1\|_F^2}
  \;=\; \frac{r_2\, \sigma_{\rmW_2}^2}{r_1\, \sigma_{\rmW_1}^2}
  \;=\; \sqrt{\frac{r_2}{r_1}}\,(1+o(1)).
  \label{eqt:sym_frob_ratio}
  \end{equation}

We now bound $\mathbb{E}\|\rmW_1 - \rmW_2 \rmP\|_F^2$ from below for an optimal column-matching permutation $\rmP$. Without loss of generality, suppose $r_1 \leq r_2$, and pad $\rmW_1$ with $r_2 - r_1$ zero columns so that both matrices have the same number of columns $r_{\max} = \max(r_1, r_2)$. Let $\rmP \in \mathbb{R}^{r_{\max} \times r_{\max}}$ be a permutation matrix that minimizes $\|\rmW_1 - \rmW_2 \rmP\|_F^2$ (equivalently, maximizes $\langle \rmW_1, \rmW_2 \rmP \rangle$). The padding does not change the squared Frobenius norm, so $\mathbb{E}\|\rmW_2 \rmP\|_F^2 = \mathbb{E}\|\rmW_2\|_F^2$. Using $\|A - B\|_F^2 = \|A\|_F^2 - 2\langle A, B \rangle + \|B\|_F^2$, we have

\begin{equation}
\mathbb{E}\|\rmW_1 - \rmW_2 \rmP\|_F^2
\;=\; \mathbb{E}\|\rmW_1\|_F^2 \;+\; \mathbb{E}\|\rmW_2\|_F^2
\;-\; 2\,\mathbb{E}\bigl[\langle \rmW_1, \rmW_2 \rmP \rangle\bigr].
\label{eqt:sym_dist_expansion}
\end{equation}

By the Cauchy--Schwarz inequality, we have

\begin{align}
\mathbb{E}\bigl[\langle \rmW_1, \rmW_2 \rmP \rangle\bigr]
&\;\leq\; \mathbb{E}\bigl[\|\rmW_1\|_F\, \|\rmW_2 \rmP\|_F\bigr] \\
&\;\leq\; \sqrt{\mathbb{E}\|\rmW_1\|_F^2 \cdot \mathbb{E}\|\rmW_2 \rmP\|_F^2} \\
&\;=\; \sqrt{\mathbb{E}\|\rmW_1\|_F^2 \cdot \mathbb{E}\|\rmW_2\|_F^2}.
\end{align}

Substituting \eqref{eqt:sym_frob_ratio}:

\begin{equation*}
\sqrt{\mathbb{E}\|\rmW_1\|_F^2 \cdot \mathbb{E}\|\rmW_2\|_F^2}
\;=\; \sqrt{\mathbb{E}\|\rmW_1\|_F^2 \cdot \sqrt{\tfrac{r_2}{r_1}}\,\mathbb{E}\|\rmW_1\|_F^2}
\;=\; \Bigl(\tfrac{r_2}{r_1}\Bigr)^{\!1/4} \mathbb{E}\|\rmW_1\|_F^2.
\end{equation*}

Plugging into \eqref{eqt:sym_dist_expansion} and using $\mathbb{E}\|\rmW_2\|_F^2 = (r_2/r_1)^{1/2}\,\mathbb{E}\|\rmW_1\|_F^2$:

\begin{align}
\mathbb{E}\|\rmW_1 - \rmW_2 \rmP\|_F^2
&\;\geq\; \mathbb{E}\|\rmW_1\|_F^2 + \Bigl(\tfrac{r_2}{r_1}\Bigr)^{\!1/2} \mathbb{E}\|\rmW_1\|_F^2 - 2\Bigl(\tfrac{r_2}{r_1}\Bigr)^{\!1/4}\mathbb{E}\|\rmW_1\|_F^2 \\
&\;=\; \biggl(1 - 2\Bigl(\tfrac{r_2}{r_1}\Bigr)^{\!1/4} + \Bigl(\tfrac{r_2}{r_1}\Bigr)^{\!1/2}\biggr)\,\mathbb{E}\|\rmW_1\|_F^2 \\
&\;=\; \biggl(1 - \Bigl(\tfrac{r_2}{r_1}\Bigr)^{\!1/4}\biggr)^{\!2}\,\mathbb{E}\|\rmW_1\|_F^2.
\end{align}

By setting $r_1 = r^*$, $\rmW_1 = \rmW^*$, we then have

\begin{equation}
\mathbb{E}\bigl[\|\rmW^* - \rmW_2 \rmP\|_F^2\bigr]
\;\geq\;
\left[
\Biggl(1 - \biggl(\frac{r_2}{r^*}\biggr)^{\!1/4}\Biggr)^{\!2}
+ o(1)
\right]
\,\mathbb{E}\bigl[\|\rmW^*\|_F^2\bigr].
\label{eqt:sym_relerr}
\end{equation}

Under these assumptions, the lower bound \eqref{eqt:sym_relerr} is strictly positive whenever $r_2 \neq r^*$, so misspecifying $r$ induces a finite relative error in the recovered factor. Thus, symmetric NMF is relatively tolerant of $r$-misspecification, because the latent factor is shared between rows and columns of $\rmS$, allowing $\sigma_{\rmW}^2$ to absorb a $\sqrt{r_2/r^*}$ rescaling. The tolerance is bounded, however, and the cost of misspecification grows monotonically in $|r_2 - r^*|$.

\subsection{Information leakage given a structured observation pattern}
\label{subsec:info_leakage}
The misspecified-$r$ bound shows that choosing the latent dimension $r$ correctly matters, even though Symmetric NMF is relatively tolerant of $r$-misspecification. In practice, however, $r$ is unknown and must be estimated from data. A standard approach is cross-validation (CV), where a subset of off-diagonal entries of $S$ is held out. The model is then fit to the remaining entries, and $r$ is chosen to minimize held-out prediction error. This section explains why this CV procedure can be unreliable for symmetric low-rank matrices.

Note that a rank-$r$ symmetric $n \times n$ matrix has far fewer effective degrees of freedom than its $\binom{n}{2}$ off-diagonal entries. Under a sufficiently informative number of observations, the held-out entries are determined by the \emph{Nystr\"om completion formula}. In the noise-free setting, a rank-$r^*$ fit to the training entries therefore recovers both the training and held-out values exactly. Consequently, the held-out CV error is identically zero at $r = r^*$, even though the held-out entries were not used during fitting. The same argument applies to any $r \geq r^*$, since the true matrix remains feasible within every higher-rank model class. The CV curve therefore plateaus at and above $r^*$ instead of attaining a unique minimum there, which makes entrywise CV unreliable for identifying the true rank.

\begin{theorem}
\label{thm:sym_nystrom}
Let $S = W W^\top$ where $W \in \mathbb{R}^{n \times r}$ has rank $r$.
Fix an index set $I \subseteq [n]$ with $|I| = r$, and define

\begin{equation}
S_I \;:=\; S[I,\,I] \;\in\; \mathbb{R}^{r \times r}
\qquad\text{(principal submatrix on }I\text{),}
\end{equation}

\begin{equation}
S_{p,\,I} \;:=\; (S_{pj})_{j \in I} \;\in\; \mathbb{R}^{r}
\qquad\text{(row of $S$ at index $p$, restricted to columns $I$).}
\end{equation}

\begin{equation}
S_{I,\,q} \;:=\; (S_{jq})_{j \in I} \;\in\; \mathbb{R}^{r}
\qquad\text{(column of $S$ at index $q$, restricted to rows $I$).}
\end{equation}

If $S_I$ is invertible, then for every $p, q \in [n]$,
\begin{equation}
S_{pq} \;=\; S_{p,\,I}^\top \, S_I^{-1} \, S_{I,\,q}.
\label{eqt:sym_nystrom}
\end{equation}
\end{theorem}

\begin{proof}
Let $W_I := W[I, :] \in \mathbb{R}^{r \times r}$ be the submatrix of $W$ formed by the rows indexed by $I$, and let $w_p \in \mathbb{R}^r$ denote the $p$-th row of $W$ (as a column vector). Then $S_{p,I} = W_I w_p$ and $S_I = W_I W_I^\top$, so $S_I$ is invertible if and only if $W_I$ is invertible. By direct computation,

\begin{align}
S_{p,\,I}^\top \, S_I^{-1} \, S_{I,\,q}
&\;=\; (W_I w_p)^\top \, (W_I W_I^\top)^{-1} \, (W_I w_q) \\
& \;=\; w_p^\top \, W_I^\top (W_I^\top)^{-1} W_I^{-1} W_I \, w_q \\
& \;=\; w_p^\top w_q \\
& \;=\; S_{pq}. \qedhere
\end{align}
\end{proof}

This identity implies that any observation set containing an invertible principal block and the corresponding cross-block entries determines the matrix entirely. This yields the following corollary.

\begin{corollary}[Information leakage]
\label{cor:sym_leakage}
Suppose that the observation set $\Omega \subseteq \binom{[n]}{2}$ of off-diagonal index pairs contains, for some index set $I \subseteq [n]$ with $|I| = r$,
\begin{itemize}
\item all off-diagonal entries within $I$:
$\{(i,j) : i, j \in I,\ i < j\} \subseteq \Omega$;
\item all cross-block entries:
$\{(i,j) : i < j,\ |\{i,j\}\cap I| = 1\} \subseteq \Omega$.
\end{itemize}
Suppose also that the diagonal entries $\{S_{ii}: i\in I\}$ are observed.
If $S_I$ is invertible, then every unobserved off-diagonal entry $S_{pq}$ with $(p, q) \notin \Omega$ is determined by the observed entries via \eqref{eqt:sym_nystrom}.
\end{corollary}

The total number of off-diagonal observations required by this observation pattern is

\begin{equation}
\binom{r}{2} + r(n - r) \;=\; nr - \frac{r(r+1)}{2}.
\end{equation}

It remains to justify that the condition ``$S_I$ is invertible'' holds except on a measure-zero set of factor matrices $W$. We use the following standard lemma.

\begin{lemma}[Polynomial zero sets]
\label{lem:sym_polyzeros}
Let $p \in \mathbb{R}[x_1, \dots, x_N]$ be a nonzero polynomial in $N$ real variables. Then its zero set $\{x \in \mathbb{R}^N : p(x) = 0\}$ has Lebesgue measure zero.
\end{lemma}




\begin{lemma}[Almost-sure invertibility]
\label{lem:sym_invertible}
Suppose the rows of $W$ are drawn i.i.d.\ from a distribution on $\mathbb{R}^r_{\geq 0}$ that has a density with respect to Lebesgue measure. Then, for every fixed index set $I \subseteq [n]$ with $|I| = r$, the submatrix $S_I = W_I W_I^\top$ is invertible with probability one.
\end{lemma}

\begin{proof}
The matrix $S_I$ is invertible if and only if $\det W_I \neq 0$. Therefore, it is sufficient to bound the probability that $\det W_I = 0$. If we treat the $r^2$ entries of $W_I$ as formal variables, by the Leibniz formula the map $W_I \mapsto \det W_I$ is a polynomial of total degree $r$ in the $r^2$ variables. The polynomial is not identically zero. Hence, by Lemma~\ref{lem:sym_polyzeros}, its zero set has Lebesgue measure zero. The joint density of the $r$ rows of $W_I$ is the product of $r$ row densities and therefore admits a density on $\mathbb{R}^{r \times r} = \mathbb{R}^{r^2}$. Hence, $\mathbb{P}(\det W_I = 0) = 0$ and $S_I$ is invertible with probability one.
\end{proof}

Corollary \ref{cor:sym_leakage} makes this information leakage explicit. Whenever $\Omega$ contains the structured pattern $\{(p,q): p < q,\ |\{p,q\}\cap I| \geq 1\}$ for some size-$r$ index set $I$ with $S_I$ invertible, every unobserved entry of $S$ can be recovered explicitly by a closed-form rational expression in the observed entries.

\subsection{Information leakage for arbitrary observation patterns}

The Nystr\"om argument in Appendix~\ref{subsec:info_leakage} gives a concrete example of information leakage, where the observed entries cover all entries incident to some size-$r$ index set $I$. We now show that the same phenomenon persists for sufficiently large random observation sets.

\begin{theorem}[Identifiability of $S$]
\label{thm:sym_sharp_ident}
Suppose that
\begin{enumerate}[label=(H\arabic*), ref=H\arabic*]
    \item \label{ass:r_less_n} $r < n$;
    \item \label{ass:enough_observations} $|\Omega| \geq nr - \binom{r}{2}$;
    \item \label{ass:full_jacobian} there exists some $W_0 \in \mathbb{R}^{n \times r}$ at which $\operatorname{rank}\bigl(J_\Omega(W_0)\bigr) = nr - \binom{r}{2}$.
\end{enumerate}

Then there exists a Lebesgue-null set $\mathcal{N} \subseteq \mathbb{R}^{n \times r}$ such that, for every $W \notin \mathcal{N}$ and every $(p, q) \notin \Omega$, the value of $S_{pq} = (WW^\top)_{pq}$ is locally determined by the observations $\{S_{ij}\}_{(i,j) \in \Omega}$ in a neighborhood of $W$. More precisely, it is given by an analytic function of those observations.
\end{theorem}

Before proving Theorem~\ref{thm:sym_sharp_ident}, we introduce two auxiliary observations used in its proof.

\paragraph{Rank of the Jacobian of $\phi$}
Let $\phi : \mathbb{R}^{n \times r} \to \operatorname{Sym}^{n\times n}$ be defined by $\phi(W) = WW^{\top}$. The polynomial map $\phi$ has the following property: For any $Q \in O(r) := \left\{ Q \in \mathbb{R}^{r \times r} : Q^{\top} Q = I_r \right\}$, we have

\[
\phi(W Q) = (WQ) (WQ)^{\top} = WQQ^{\top} W^{\top} = WW^{\top} = \phi(W)
\]

Thus, $\phi$ is constant on each orbit $\mathcal{O}_W := \left\{ WQ: Q \in O(r)\right\}$. The dimension of this orbit is determined by the dimension of $O(r)$. Differentiating $Q^{\top}Q = I_r$ at $Q = I_r$, with $Q(t) = I_r + tA + O(t^2)$ for some matrix $A$, gives

\[
I_r = Q(t)^{\top} Q(t) = I_r + t(A + A^{\top}) + O(t^2)
\]

forcing $A + A^{\top} = 0$. This implies that the tangent space to $O(r)$ at $I_r$ is the space of $r \times r$ skew-symmetric matrices, which has dimension $\binom{r}{2}$. For $W$ of rank $r$, the map $Q \mapsto WQ$ from $O(r)$ to $\mathcal{O}_W$ is injective, so $\mathcal{O}_W$ has dimension $\binom{r}{2}$. The differential of $\phi$ at $W$ is

\[
d \phi_W(V) = \frac{d}{d\epsilon} \Big|_{\epsilon=0} \phi(W + \epsilon V) = \frac{d}{d\epsilon}\Big|_{\epsilon=0} \left[ WW^{\top} + \epsilon (VW^{\top} + WV^{\top}) + O(\epsilon^2) \right] = VW^{\top} + WV^{\top}
\]

for each direction $V \in \mathbb{R}^{n\times r}$.

We now show that, for any $W \in \mathbb{R}^{n \times r}$ of rank $r$,

\[
\operatorname{rank}(d \phi_W) = nr - \binom{r}{2}.
\]

We first consider the canonical case $W_0 := \begin{pmatrix} I_r \\ 0 \end{pmatrix} \in \mathbb{R}^{n\times r}$. For any $V = \begin{pmatrix} V_1 \\ V_2 \end{pmatrix}$ where $V_1 \in \mathbb{R}^{r\times r}$ and $V_2 \in \mathbb{R}^{(n-r) \times r}$, we have

\[
d \phi_{W_0}(V) = V W_0^{\top} + W_0 V^{\top} = \begin{pmatrix} V_1 + V_1^{\top} & V_2^{\top} \\ V_2 & 0 \end{pmatrix}
\]

In the canonical case, the image of the differential consists of all symmetric $n \times n$ matrices of this form. Its dimension is therefore

\[
\binom{r+1}{2} + r(n-r) = \frac{r(r+1)}{2} +nr - r^2 = nr - \frac{r(r-1)}{2} = nr - \binom{r}{2}
\]

where the first term $\binom{r+1}{2}$ is the dimension of the top-left symmetric block $V_1 + V_1^{\top}$, and the second term $r(n-r)$ is the dimension of $V_2$. 

For a general $W$ of rank $r$, we can choose an orthogonal $U \in \mathbb{R}^{n \times n}$ such that $U^{\top} W = \begin{pmatrix} W_1 \\ 0 \end{pmatrix}$ with $W_1 \in \mathbb{R}^{r \times r}$ invertible. Then

\[
U^{\top} d\phi_W(V) U = (U^{\top}V)(U^{\top}W)^{\top} + (U^{\top} W)(U^{\top}V)^{\top} = d\phi_{U^{\top}W}(U^{\top}V).
\]

Since $U$ is orthogonal, this implies that $\operatorname{rank}(d\phi_{W}) = \operatorname{rank}(d\phi_{U^{\top}W})$.

Substituting $V \mapsto VW_{1}^{-\top}$ in $d \phi_{U^{\top}W}$ we have

\[
d\phi_{U^{\top}W}(VW_1^{-\top})= (VW_1^{-\top})(U^{\top}W)^{\top} + (U^{\top}W)(VW_1^{-\top})^{\top} = V \begin{pmatrix} I_r & 0 \end{pmatrix} + \begin{pmatrix} I_r \\ 0 \end{pmatrix} V^{\top} = d\phi_{W_0}(V)
\]

Thus, for any $W \in \mathbb{R}^{n\times r}$ of rank $r$
\[
\operatorname{rank}(d \phi_W) = nr - \binom{r}{2}.
\] 

By rank-nullity, the kernel of $d \phi_W$ has dimension $\binom{r}{2}$. We have also shown that the tangent space $\{ WA : A^{\top} = - A\}$ is a $\binom{r}{2}$-dimensional subspace of this kernel, since $\phi$ is constant on $\mathcal{O}_W$. Combining these, we can conclude that the inclusion is an equality, i.e., the kernel of $d \phi_W$ is exactly the tangent space to the orbit $\mathcal{O}_W$ at $W$.

\paragraph{The Implicit Function Theorem.}

We use the following form of the implicit function theorem. Let $F : \mathbb{R}^N \to \mathbb{R}^M$ be $C^1$, and let $x_0 \in \mathbb{R}^N$ be a point at which $dF_{x_0}$ has rank $M$ (so $F$ is a submersion at $x_0$, and $N \geq M$). Then there exists an open neighborhood $U$ of $x_0$ in $\mathbb{R}^N$ such that the level set $F^{-1}(F(x_0)) \cap U$ is the image of a $C^1$ embedding from an open subset of $\mathbb{R}^{N-M}$ --- i.e., an $(N-M)$-dimensional smooth submanifold of $\mathbb{R}^N$.

\begin{proof}





We first construct $\mathcal{N}$ as the union of two measure-zero sets, $\mathcal{N}_1$ and $\mathcal{N}_2$, each excluding a distinct exceptional case. The first exceptional case is that the Jacobian has rank deficiency. By \ref{ass:full_jacobian}, there exists a square submatrix 
$\mathcal{J}(W) \in \mathbb{R}^{\widetilde{n} \times \widetilde{n}}$
of $J_{\Omega}(W)$ such that $\det \mathcal{J}(W) \neq 0$, where $\widetilde{n} := nr - \binom{r}{2}$. Since each entry of $J_{\Omega}(W)$ is linear in $W$ by definition, the map $W \mapsto \det \mathcal{J}(W)$ is a polynomial in the entries of $W$. As this polynomial is not identically zero, Lemma~\ref{lem:sym_polyzeros} implies that

\[
\mathcal{N}_1 := \{ W \in \mathbb{R}^{n \times r} : \det \mathcal{J}(W) = 0 \}
\]

has measure zero. 

A second exceptional case is rank deficiency of the factor. Specifically, if $\operatorname{rank}(W) < r$, then the previous rank result for the differential does not apply. We have

\[
\operatorname{rank}(W) = r \iff \text{some $r$ rows of $W$ are independent} \iff \exists\, I \subseteq [n],\ \vert I \vert = r,\ \det W[I, :] \neq 0
\]

Here, $W[I,:]$ denotes the $r \times r$ submatrix formed by the rows indexed by $I$. Negating the statement gives
\[
\operatorname{rank}(W) < r \iff \forall\, I \subseteq [n],\ \vert I \vert = r,\ \det W[I, :] = 0
\]

To combine the $\binom{n}{r}$ possible index sets into a single polynomial, we use the standard sum-of-squares trick:
\[
(\forall I, \det W[I, :] = 0) \iff \sum_{\vert I \vert = r} \left( \det W[I,:] \right)^2 = 0
\]

Since a sum of non-negative real numbers vanishes if and only if every term vanishes, we have

\[
\left\{ W : \operatorname{rank}(W) < r \right\} = \left\{ W : \sum_{\vert I \vert = r} \left( \det W[I, :] \right)^2 = 0\right\}
\]

Now define 
\[ 
P(W) = \sum_{\vert I \vert = r} \left( \det W[I, :] \right)^2 
\]

Then $P$ is a polynomial in the entries of $W$ and each determinant $\det W[I,:]$ is a polynomial of degree $r$. Therefore, $P$ has degree $2r$. To see that $P$ is not identically zero, evaluate it at the canonical matrix $\begin{pmatrix} I_r \\ 0 \end{pmatrix}$. For $I = \{ 1, 2, \ldots, r\}$, we have $W[ I, :] = I_r \Rightarrow \det W [I, :] = 1$. For any other $I$, at least one row of $W[I, :]$ comes from the zero block, so $\det W[I,:] = 0$. Hence, $P\!\left(\begin{pmatrix} I_r \\ 0 \end{pmatrix}\right) = 1^2 + 0 + \cdots + 0 = 1 \neq 0 \Rightarrow P \not\equiv 0$. Then by Lemma~\ref{lem:sym_polyzeros}, $\mathcal{N}_2 = \{W: P(W) = 0 \}$ has measure zero.

Now fix $W \not\in \mathcal{N} := \mathcal{N}_1 \cup \mathcal{N}_2$ , and let $\widetilde{\Omega} \subseteq \Omega$ be the index set witnessing $\operatorname{rank}(J_{\Omega}(W)) = nr - \binom{r}{2}$. Define the restriction map
\[ 
\phi_{\widetilde{\Omega}} : \mathbb{R}^{n \times r} \to \mathbb{R}^{nr - \binom{r}{2}}, \quad W \mapsto \{(WW^{\top})_{ij}\}_{(i,j) \in \widetilde{\Omega}} 
\]

Its differential at $W$ has matrix $J_{\widetilde{\Omega}}(W)$, which has full row rank $nr - \binom{r}{2}$. Hence, $\phi_{\widetilde{\Omega}}$ is a submersion at $W$. By the implicit function theorem, there is an open neighborhood $U \subseteq \mathbb{R}^{n \times r}$ of $W$ such that the solution set

\[\mathcal{S} := \left\{ W' \in U : \phi_{\widetilde{\Omega}}(W') = \phi_{\widetilde{\Omega}}(W)\right\}\]

is a smooth submanifold of $U$ of dimension

\[ \dim \mathcal{S} = nr - \left( nr - \binom{r}{2} \right) = \binom{r}{2}. \]

The tangent space at $W$ is $T_{W} \mathcal{S} = \ker J_{\widetilde{\Omega}}(W)$. Recall the orbit of $W$,
\[
\mathcal{O}_W := \{ WQ : Q \in O(r) \}
\]

For every $Q \in O(r)$, we have

\[
\phi_{\widetilde{\Omega}}(WQ) = \{ (WQQ^{\top} W^{\top})_{ij} \}_{(i,j) \in \widetilde{\Omega}} = \{ (WW^{\top})_{ij} \}_{(i,j) \in \widetilde{\Omega}} = \phi_{\widetilde{\Omega}}(W)
\]

Therefore $WQ \in \mathcal{S}$ whenever $WQ \in U$. Thus $\mathcal{O}_{W} \cap U \subseteq \mathcal{S}$. Moreover, the map $Q \mapsto WQ$ from $O(r)$ to $\mathcal{O}_W$ is smooth, since it is linear in $Q$.  It is injective when $\operatorname{rank}(W) = r$, which holds because $W \not\in \mathcal{N}_2$. Since $\dim O(r) = \binom{r}{2}$, the image $\mathcal{O}_W$ is a smooth submanifold of $\mathbb{R}^{n\times r}$ of dimension $\binom{r}{2}$ near $W$. Differentiating the orbit map at $Q = I_r$ along a path $Q(t) \in O(r)$ with $Q(0) = I_r$ and $\frac{d Q(t)}{dt} \big|_{t=0} = A$ gives

\[
\frac{d }{dt} \big|_{t=0} WQ(t) = WA 
\]

Therefore $T_{W} \mathcal{O}_W = \{ WA : A^{\top} = -A \}$. Since $\ker J_{\widetilde{\Omega}}(W) = \{ WA : A^{\top} = -A\}$, we have $T_W \mathcal{O}_W = T_W \mathcal{S}$. Both $\mathcal{O}_W \cap U$ and $\mathcal{S}$ are smooth $\binom{r}{2}$-dim manifolds passing through $W$, with $\mathcal{O}_W \cap U \subseteq \mathcal{S}$ and identical tangent spaces at $W$. Therefore, the inclusion $\mathcal{O}_W \hookrightarrow \mathcal{S}$ is a local diffeomorphism near $W$. After possibly shrinking $U$, we have
\[
\mathcal{S} = \mathcal{O}_W \cap U
\]

Therefore, every $W'$ near $W$ satisfying
\[
\phi_{\widetilde{\Omega}}(W') = \phi_{\widetilde{\Omega}}(W) 
\]

must have the form $W' = WQ$ for some $Q \in O(r)$. Hence,

\[
W' (W')^{\top} = (WQ)(WQ)^{\top} = WQQ^{\top} W^{\top} = WW^{\top} 
\]

This equality holds entry-wise, and in particular for every $(p, q) \not\in \Omega$, 
\[ 
S_{pq} = (WW^{\top})_{pq} = (W' (W')^{\top})_{pq} 
\]

Finally, since $\phi_{\widetilde{\Omega}}$ is polynomial and has full-rank differential at $W$, the implicit function theorem also gives local real-analytic dependence on the observed coordinates. Therefore, the recovered missing entries are locally smooth functions of the observed entries. This smoothness implies that identifiability persists under perturbation of the observations.
\end{proof}

\paragraph{Remark.} For random sampling of $\Omega$, the relevant phenomenon is matrix completion \citep{Candes2009,Gross2011}. Standard results imply that, under incoherence and sufficiently many uniformly sampled entries, a rank-$r$ matrix is uniquely determined by $O(nr \log^c n)$ observed entries with high probability. Thus, even randomly held-out entries may not be genuinely independent validation targets. Once the training set crosses the matrix-completion threshold, the held-out entries are already determined by the training entries.


\subsection{Cross-validation protocol}
\label{subsec:cv_protocol}
To select the representation dimension $r$, we use a cross-validation protocol calibrated to symmetric similarity matrices. The protocol addresses two issues that affect entrywise CV on low-rank matrices: held-out entries can be implicitly recovered from the training entries through matrix-completion, and the effective training fraction depends on the number of folds. 

\paragraph{Calibration.} From the observed similarity matrix $S$, the protocol determines two quantities without reference to any specific factorization model:
\begin{itemize}
\item a spectral cutoff $k_{\mathrm{cut}}$, the number of leading eigendirections of $S$ that remain stable under random off-diagonal subsampling, and
\item an operating sampling probability $p^*$, the smallest sampling probability at which the top-$k_{\mathrm{cut}}$ spectral structure of the subsampled matrix recovers at least a fraction $1-\delta$ of the reference spectral mass.
\end{itemize}
We use $\delta = 0.10$ throughout, so $p^*$ is the smallest probability at which $\geq 90\%$ of the top-$k_{\mathrm{cut}}$ spectral mass is preserved. Both $k_{\mathrm{cut}}$ and $p^*$ are model-independent calibration quantities determined entirely by $S$.

\paragraph{Fold-invariant cross-validation.} Given $p^*$ and a chosen number of folds $k_{\mathrm{cv}}$, we sample an outer observation mask over unordered off-diagonal pairs, retaining each pair independently and mirroring the mask symmetrically, with probability
\begin{equation}
p_{\mathrm{cv}} =
\min\left\{
0.95,\,
p^*\,\frac{k_{\mathrm{cv}}}{k_{\mathrm{cv}} - 1}
\right\}.
\label{eq:fold_invariant_p}
\end{equation}
The cap prevents the outer-mask probability from becoming degenerate when the fold-invariant value would be close to or above one. Ordinary $k_{\mathrm{cv}}$-fold cross-validation is then performed within this observed pool. Folds are assigned at the level of unordered off-diagonal pairs and mirrored symmetrically. When the cap is inactive, each fold's expected training fraction equals $p^*$ regardless of $k_{\mathrm{cv}}$, so the recovered rank $\hat r$ is less sensitive to the arbitrary fold-count choice.

\paragraph{Implementation in this work.} We used $k_{\mathrm{cv}} = 5$ folds and $\delta = 0.10$ for all analyses. For each dataset, the calibrated $p^*$ used to construct the outer mask is reported in Table~\ref{tab:datasets}. With these values and \eqref{eq:fold_invariant_p}, the masking construction is fully specified.

\begin{figure}[h]
    \centering
    \includegraphics[width=\linewidth]{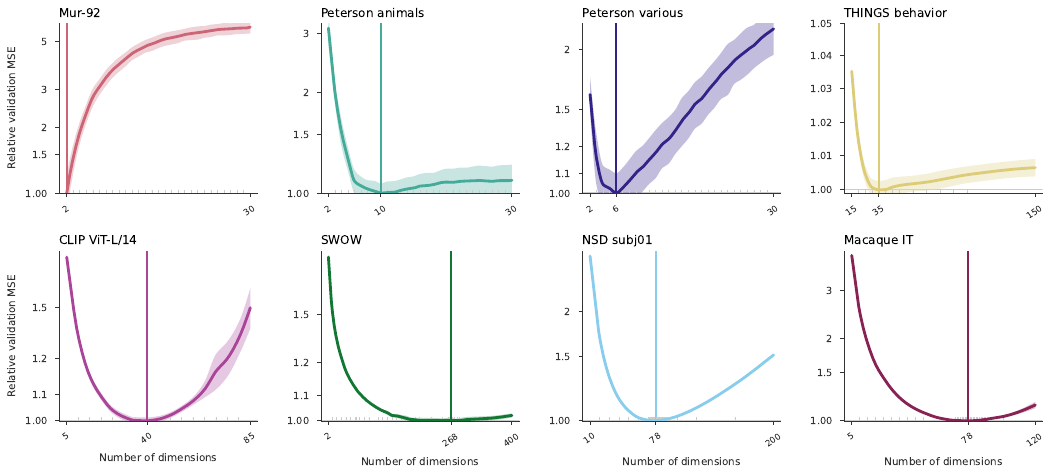}
    \caption{\textbf{Cross-validated rank selection across datasets.}
    Relative validation MSE as a function of the candidate number of dimensions $r$ for each dataset (Mur-92, Peterson (Animals), Peterson (Various), THINGS (Behavior), CLIP ViT-L/14, SWOW, NSD (subj01), THINGS (Macaque)). Curves are normalized to the minimum validation error of each dataset. The vertical line in each panel marks the cross-validated rank $r^*$ reported in Table~\ref{tab:datasets}. The shaded band shows the standard deviation across 5 folds $\times$ 5 repeats. The outer observation mask is constructed from the dataset-specific calibrated $p^*$ via \eqref{eq:fold_invariant_p}.}
    \label{fig:cv_curves}
\end{figure}

\section{Dataset specifics}

\begin{table}[!htbp]
  \centering
  \small
  \setlength{\tabcolsep}{4pt}
  \caption{\textbf{Overview of datasets used for SRF evaluation.} $N$ denotes the number of items, Coverage\ the percentage of off-diagonal item pairs represented in the similarity matrix, $p^*$ the sampling probability used to construct the outer mask \eqref{eq:fold_invariant_p}, and $r^*$ the rank selected by 5-fold cross-validation. $R^2$ measures reconstruction quality. For SWOW, reconstruction quality is reported as link-prediction AUC. CV reliability denotes mean cross-validated split-half dimension reliability across 30 random seeds. A dash indicates that the metric is not applicable.}
  \label{tab:datasets}
  \begin{tabular}{llccccccr}
  \toprule
  Dataset & Modality & $N$ & Coverage\ (\%) & $p^*$ & $r^*$ & $R^2$ & AUC & CV reliability \\
  \midrule
  Mur-92             & Behavior   & 92       & 100  & 0.31 &   2 & 0.67 & -    & >0.99 \\
  Peterson Animals & Behavior   & 120      & 100  & 0.70 &  10 & 0.90 & -    & >0.99 \\
  Peterson Various & Behavior   & 120      & 100  & 0.76 &   6 & 0.86 & -    & >0.99 \\
  CLIP ViT-L/14      & Model      & 1{,}854  & 100  & 0.71 &  40 & 0.88 & -    & 0.94  \\
  Macaque IT   & Neural     & 1{,}854  & 100  & 0.73 &  78 & 0.98 & -    & 0.91  \\
  THINGS odd-one-out  & Behavior   & 1{,}854  & 99.9 & 0.65 &  35 & 0.61 & -    & 0.96  \\
  SWOW               & Behavior   & 8{,}647  & 1.27 & 0.53 & 268 & -    & 0.95 & 0.94  \\
  NSD (subj01)       & Neural     & 9{,}841  & 100  & 0.62 &  78 & 0.99 & -    & 0.82  \\
  \bottomrule
  \end{tabular}
\end{table}

\clearpage

\section{Statistical power by dimension variance}

\begin{figure}[h]
    \centering
    \includegraphics[width=0.9\linewidth]{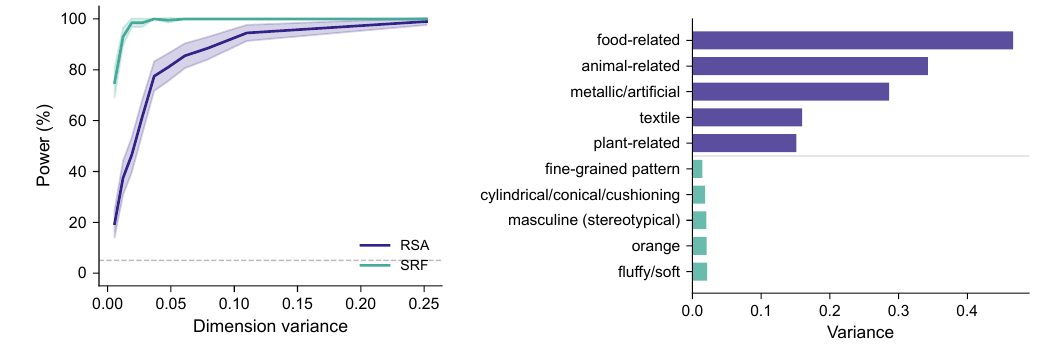}
    \caption{\textbf{SRF advantage is most pronounced for low-variance dimensions.}
    \textbf{a},~Statistical power of RSA and SRF as a function of dimension variance in the SPoSE embedding at SNR\,=\,1.0, with 95\% confidence intervals. RSA power drops sharply for dimensions with low variance, which contribute little to overall similarity and are therefore difficult to detect when tested against the full matrix. SRF maintains high detection rates across the full variance range by isolating each dimension before testing.
    \textbf{b},~Example SPoSE dimensions sorted by variance. High-variance dimensions (top) capture broad semantic categories such as \textit{food-related} or \textit{animal-related}, while low-variance dimensions (bottom) capture fine-grained properties such as \textit{fluffy/soft} or \textit{orange}. The dashed line separates the five highest-variance from the five lowest-variance dimensions.}
    \label{fig:variance_quartile}
\end{figure}

\section{Semantic similarity benchmarks}

\begin{figure}[h]
    \centering
    \includegraphics[width=0.9\linewidth]{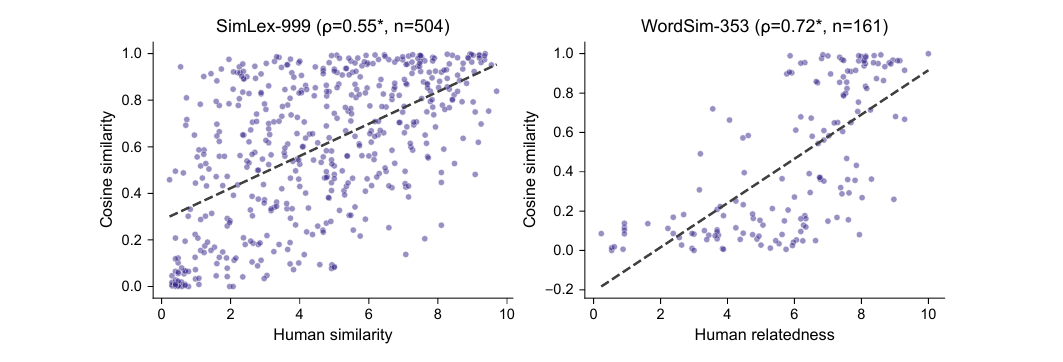}
    \caption{\textbf{SRF embeddings from word associations predict human similarity and relatedness judgments.}
    Cosine similarity between SRF embedding vectors (derived from the Small World of Words dataset) is plotted against human ratings from two independent benchmarks.
    \textbf{a},~SimLex-999 \citep{Hill2015} measures genuine semantic similarity, where word pairs are rated by how similar their meanings are (e.g., \textit{smart}--\textit{intelligent}).
    \textbf{b},~WordSim-353 \citep{Finkelstein2002} measures broader semantic relatedness, including associative relationships (e.g., \textit{car}--\textit{road}).
    Spearman correlations ($\rho$) and number of word pairs with available embeddings ($n$) are shown. Dashed lines show linear fits.}
    \label{fig:similarity_benchmarks}
\end{figure}

\end{document}